%% file: revised.tex
\documentclass[journal,twocolumn]{IEEEtran}

\usepackage{amsmath,amsfonts,amssymb,paralist,latexsym,mathtools,
epsfig,color,rotating,times,float,
subcaption,algorithm,verbatim,enumitem,bm,graphicx,listings,url,multirow}
\usepackage{epstopdf}
\usepackage{tikz}
\usepackage[normalem]{ulem}

\include{mycommands}

\begin{document}

\title{Multi-kernel Passive  Stochastic Gradient Algorithms and Transfer Learning}

\author{Vikram Krishnamurthy, {\em Fellow IEEE} and    George Yin, {\em Fellow IEEE}
  \thanks{Vikram Krishnamurthy is with the School of
 Electrical \& Computer
Engineering,  Cornell University, NY 14853, USA.
(vikramk@cornell.edu).  G. Yin is with Department of Mathematics, University of Connecticut, Storrs, CT 06269-1009, USA. (gyin@uconn.edu).
 This research was supported by U.S.\ Army Research Office under grant  W911NF-19-1-0365, National Science Foundation under grant 1714180, and Air Force Office of Scientific Research under grant FA9550-18-1-0268.
}}


\maketitle

\begin{abstract}
This paper develops a novel passive stochastic gradient algorithm.
  In  passive stochastic approximation, the stochastic gradient algorithm does not have control over the location where noisy gradients of the  cost function are evaluated. Classical  passive  stochastic gradient algorithms use a kernel that approximates a Dirac delta to weigh the gradients based on how far they are evaluated from the desired point. In this paper we construct a multi-kernel passive stochastic gradient algorithm. The algorithm performs substantially  better in high dimensional problems and incorporates variance reduction. We analyze the weak convergence of the multi-kernel algorithm and its rate of convergence. In numerical examples, we study the multi-kernel version of the passive least mean squares (LMS) algorithm for transfer learning
  to compare the performance
  with the classical passive version.
\end{abstract}

{\em Keywords}.
stochastic gradient algorithm, weak convergence,   stochastic sampling,  variance reduction, passive LMS, transfer learning,  Bernstein von-Mises theorem

\allowdisplaybreaks

\section{Introduction} \label{sec:intro}

\IEEEPARstart{S}{\lowercase{uppose}} an  agent evaluates noisy gradients of a cost function $\Cost(\cdot)$. At each time $k$,  the agent samples a random point $\th_k  \in \reals^\thdim$  from the probability density   $\belief(\cdot)$   and then
evaluates the noisy gradient $\D_\th \cost_k(\th_{k})$ of the true gradient
$\nabla \Cost(\th_k)$.
  By intercepting the dataset $\{\th_k,\D_\th \cost_k(\th_{k}),k=1,2,\ldots\}$ from the agent,  how can we estimate a local stationary point of  the cost $\Cost(\cdot)$?


 It is well known  \cite{Rev77,HN87,NPT89,YY96} that
given the dataset $\{\th_k, \D_\th \cost_k(\th_k),k=1,2,\ldots\}$, we can estimate a local stationary point of $\Cost(\cdot)$ using the following classical {\em passive} stochastic gradient algorithm:
\beq
\eth_{k+1} = \eth_k - \step   \kerneln(\frac{\th_{k} - \eth_k}{\kernelstep}) \,\D_\th \cost_k(\th_{k}) ,
\quad \th_k \sim \belief
\label{eq:passive}  \eeq
where step size  $\step$ is a small positive constant.
Note that  (\ref{eq:passive}) is a  {\em passive} stochastic gradient   algorithm since the gradient is not evaluated at $\eth_\dtime$ by the algorithm; instead the noisy
gradient $ \D_\th \cost_k(\th_{k})$ is evaluated at a  random point  $\th_{\dtime}$ chosen by the agent from probability density  $\belief$.

The key construct in the passive gradient algorithm (\ref{eq:passive})  is the kernel function
   $\kernel(\cdot)$. This kernel function $\kernel(\cdot)$ is chosen  such that it decreases monotonically to zero as any component of the argument increases to infinity, and
   \beq \label{eq:kernel_properties} \kernel(\th) \geq 0, \quad \kernel(\th) = \kernel(-\th), \quad  \int_{\reals^\thdim} \kernel(\th) d\th = 1.
   \eeq
 The parameter  $\kernelstep$ that appears in the kernel in  (\ref{eq:passive}) is a small positive constant.
Examples of the kernel $\kernel(\cdot)$ include  the  multivariate normal  $\normald(0,\sigma^2I_{\thdim})$ density\footnote{With suitable abuse of notation, we use $\normal$ for both normal density and distribution; the distinction is clear from  the context.}  with $\sigma = \kernelstep $, i.e.,
$$ \kerneln\bigl(\frac{\th}{\kernelstep}\bigr) = (2 \pi)^{-\thdim/2}  \kernelstep^{-\thdim} \exp \bigl (- \frac{\|\th\|^2}{2 \kernelstep^2}\bigr),
$$
which is essentially  like
a Dirac delta centered at 0 as $\kernelstep \rightarrow 0$.

The kernel $\kernel(\cdot)$ in (\ref{eq:passive})  weights the usefulness of the gradient $\D_\th\cost_\dtime(\th_\dtime)$ compared to the required gradient estimate
$\D_\eth
\cost_\dtime(\eth_\dtime)$.
If $\th_\dtime$ and $\eth_\dtime$ are far apart, 
kernel $\kernel((\th_\dtime-\eth_\dtime)/\kernelstep)$  will be  small. Then only a small proportion of the gradient estimate $\D_\th \cost_\dtime(\th_\dtime)$ is added to the passive algorithm. On the other hand, if $
\eth_\dtime = \th_\dtime$, then $\kerneln(\cdot) = 1$ and  (\ref{eq:passive}) becomes a standard stochastic gradient algorithm.

\subsection*{Main Idea: Multi-kernel Passive Algorithm}
    For high dimensional problems (large  $\thdim$), the passive algorithm (\ref{eq:passive}) can take a large number of iterations to converge.
    This is because with high probability, the kernel $\kernel(\th_k,\eth_k)$ will be close to zero and so updates of $\eth_k$  will occur very rarely.
Further, from an implementation point
of view, for small $\kernelstep$, the scale factor $\kernelstep^{-\thdim}$ in (\ref{eq:passive}) blows up  for moderate to large $\thdim$; to compensate, a very small step size $\step$ needs to be used.
    Also algorithm (\ref{eq:passive}) is sensitive to the choice of
    the probability density $\belief(\cdot)$  from which the $\th_k$ are sampled to generate $\D_\th \cost_\dtime(\th_\dtime)$.
Moreover,
there is strong motivation to introduce variance reduction in the algorithm.

{\em Our main idea is to propose and analyze a
two time step, multi-kernel, variance reduction algorithm motivated by importance sampling.}
Apart from the ability to deal with high dimensional problems, the algorithm achieves variance reduction in the samples.

Assume that  at each time $k$ we are given a sequence of noisy gradients
$\{\D_\th \cost_{k,l}(\th_{k,l}), l = 1,\ldots, \numparticles\}$ which are  unbiased estimates\footnote{In Sec.\ref{sec:weak}, we  make the dependence of
$\D_\th
\cost_{k,l}(\th)$
on $k$ and $l$ more general in terms of additive measurement noise that is i.i.d. in $l$ and mixing in $k$.}
of $\nabla \Cost(\th_{k,l})$
Here the points $\th_{k,l}$ are sampled i.i.d. from density  $\belief(\cdot)$.
Given $\{\th_{k,l}, \D_\th \cost_{k,l}(\th_{k,l}), l = 1,\ldots, \numparticles\}$ at each time $k$, we propose the following multi-kernel passive algorithm with
step size~$\step$:
\beq
\boxed{\begin{split}
  \eth_{\dtime+1} &= \eth_\dtime - \step \, \frac{\sum_{i=1}^\numparticles   \pdf(\eth_\dtime|\th_{\dtime,i}) \D_\th \cost_{\dtime,i}(\th_{\dtime,i}) }{\sum_{l=1}^\numparticles \pdf(\eth_\dtime|\th_{\dtime,l})}, \quad \th_{k,i} \sim \belief
       \end{split}} \label{eq:mcmcirl}
     \eeq
     In (\ref{eq:mcmcirl}),  we choose the conditional probability density function
     \beql{eq:cond-al}\pdf(\eth|\th) = \pdf_\obsnoise(\th - \eth)
     \eeq
      where $ \pdf_\obsnoise(\cdot)$  is a symmetric density about 0 with variance $O(\kernelstep^2)$.
      For example, we can choose
      $ \pdf_\obsnoise(\cdot)$ to be the density of normal distribution
       $\normald(0,\kernelstep^2I_\thdim)$
      or
      an~$\thdim$-variate Laplace density with scale parameter $\kernelstep$:
      \beq \pdf_\obsnoise(\th-\eth) =
      \frac{1}{(2  \kernelstep)^{\thdim}} \exp \bigl (- \frac{\|\th-\eth\|_1}{\kernelstep}\bigr).
      \label{eq:laplace}
\eeq

    For notational convenience, for each $\eth_\dtime$, denote the normalized weights in (\ref{eq:mcmcirl})  at time $k$ as
\beq  \weight_{\dtime,i}(\eth_\dtime)  =   \frac{\pdf(\eth_\dtime|\th_{\dtime,i}) }{\sum_{l=1}^\numparticles \pdf(\eth_\dtime|\th_{\dtime,l})}  , \quad i = 1,\ldots, \numparticles   \label{eq:weights} \eeq
     Then these
     $\numparticles $ normalized weights  qualify  as
   symmetric kernels in the sense of (\ref{eq:kernel_properties}). Thus  algorithm (\ref{eq:mcmcirl}) can be viewed as  a multi-kernel  passive stochastic approximation
   algorithm.

\subsection*{Discussion}
(i) The key idea behind the multi-kernel algorithm \eqref{eq:mcmcirl} is as follows: using importance sampling arguments and averaging theory (Theorem \ref{thm:conv1} below),  as $\numparticles \rightarrow \infty$, the RHS of \eqref{eq:mcmcirl} yields
\begin{equation}\barray \disp
 \sum_{i=1}^\numparticles \weight_{\dtime,i}(\al) \,
\D_\th
\cost_{\dtime,i}(\th_{\dtime,i}) \ad\xrightarrow{\text{w.p.1}} \int_{\reals^\thdim} \nabla \Cost(\th) \, \pdf_\stepa(\th|\al) \, d\th
 \\
\ad = \E\{\D_\th \cost_{\dtime,l}
(\th) | \eth_\dtime = \eth\}
\label{eq:sir}\earray
\end{equation}
where $\pdf_\stepa(\th|\eth_k) \propto \belief(\th) \,\pdf(\eth_k|\th)$ denotes the posterior  density of $\th$ given $\eth_k$ where likelihood  $\pdf(\eth|\th)$ is evaluated in \eqref{eq:cond-al}. Thus the RHS
of \eqref{eq:mcmcirl} mimics a simulation based Bayesian update.
It is this posterior $\pdf_\stepa(\th|\eth_k)$ that gives the gradient algorithm \eqref{eq:mcmcirl} improved performance compared to the classical
passive algorithm \eqref{eq:passive}.
Note that the  conditional expectation $ \E\{\D_\th \cost_{\dtime,l}(\th) | \eth_\dtime \}$  always has smaller variance than
$ \D_\th \cost_{\dtime,l}(\th)$; therefore variance reduction is achieved in  the multi-kernel algorithm
(\ref{eq:mcmcirl}). \\
(ii)  Unlike the classical passive algorithm (\ref{eq:passive}), the multi-kernel algorithm (\ref{eq:mcmcirl}) does not have
the problematic term $O(\kernelstep^{-\thdim})$. Indeed, we can choose $\pdf(\eth|\th)
\propto  \pdf_\obsnoise(\th - \eth)$ in \eqref{eq:cond-al} since the scale factors cancel out. So from a practical point of view,  the multi-kernel algorithm has better numerical properties and does not need fine tuning
 the step size.
\\
(iii)
Throughout this paper we consider constant step size
algorithms, i.e., $\step$ is a fixed constant (instead of a decreasing step size).  This facilitates  estimating (tracking) parameters  that evolve over time. Due to the constant step size, the appropriate notion of convergence is weak convergence \cite{KY03,EK86,Bil99}.\\
(iv) Sec.\ref{sec:weak} and \ref{sec:rate}  analyze weak convergence  and asymptotic
convergence rate of the multi-kernel algorithm. We show that the multi-kernel algorithm has the same asymptotic convergence rate as a classical stochastic approximation algorithm.
In comparison, the classic passive stochastic gradient algorithm
 needs to ``balance'' the stepsize $\e$ with kernel step size $\mu$; indeed \cite{YY96}
shows that the convergence rate of classical passive stochastic gradient algorithm is always slower than that of the classical stochastic gradient algorithm.
  Thus, the
 multi-kernel algorithm always has faster rate of convergence than the passive algorithm  \eqref{eq:passive}.

\subsection*{Examples}

We refer to \cite{Rev77,HN87,NPT89,YY96} for the analysis and applications of  passive stochastic  gradient algorithms.  \cite{HN87} illustrates the classical passive gradient algorithm  on a real data set in forensic medicine for  estimating the mean age from weight of unknown corpses. \cite{YY96} presents  a detailed application in parameter estimation
    of chemical processing plants.
Recently, we have developed inverse  reinforcement learning  \cite{KY20} using simulated annealing versions of passive stochastic gradient algorithms.
In addition to these examples,
from an application point of view, the above setup can be viewed in a passive (or adversarial) framework.  We  passively intercept (view) the dataset $\{\th_{k,l}, \D_\th \cost_{k,l}(\th_{k,l}), l=1\ldots,\numparticles\}$ generated by $\numparticles$ independent agents. By intercepting the dataset, how can we estimate a stationary point of the cost $\Cost(\cdot)$?   Note that we have no control over where the agent evaluates the noisy gradients.

Another  application is discussed  in Sec.\ref{sec:mis}  where at each time $k$ we request the evaluation of the gradient at point $\eth_k$. However, the agent evaluates the gradient at a mis-specified point $\th_k$. Unlike classical stochastic gradient algorithms where only the gradient evaluated at $\eth_k$ is corrupted by noise, here both the evaluation point $\eth_k$ (noisy value of $\th_k)$ and the gradient value
$\D_\th \cost_k(\th_k)$  are corrupted by noise.

Finally, Sec.\ref{sec:numerical} discusses an application of passive stochastic approximation involving transfer learning and the passive least mean squares algorithm.
Transfer learning refers to using knowledge gained in one domain to learn in another domain. For our purposes,
we show how to estimate  the solution of a stochastic optimization problem by observing the training data of another stochastic optimization problem. In effect the knowledge gained by solving one stochastic optimization problem is transferred to solving another problem.

\subsection*{Organization}
Sec.\ref{sec:informal} discusses the main intuition behind the passive algorithm using the ordinary differential equations obtained via stochastic averaging.  Sec.\ref{sec:weak} gives a formal weak convergence proof of the multi-kernel passive algorithm (\ref{eq:mcmcirl}).
Sec.\ref{sec:rate} characterizes  the rate of convergence of the multi-kernel algorithm.
Sec.\ref{sec:mis} discusses a   mis-specified algorithm where the gradient is evaluated at a point $\th_\dtime$ that is a corrupted value of $\eth_\dtime$. Finally, Sec.\ref{sec:numerical} considers passive least mean squares (LMS) algorithms for transfer learning;   we  compare  in numerical examples the convergence of the classical passive LMS algorithm versus the multi-kernel passive LMS algorithm.

\section{Informal Convergence Analysis of Passive Algorithms}
\label{sec:informal}

The main intuition behind the passive algorithms is straightforwardly  captured by averaging theory. We discuss this below.
As is well known
 \cite{KY03}, a classical fixed step size  stochastic gradient  algorithm
converges weakly  to a \textit{deterministic}  ordinary differential equation (ODE) limit; this is the basis of the so-called ODE approach  for analyzing stochastic gradient algorithms.
Weak convergence  is a function space generalization of convergence in distribution.
As is typically done in weak convergence analysis,
we first represent the sequence of estimates $\{\eth_k\}$  generated by the passive algorithm as a continuous-time random process. This is done by constructing the continuous-time trajectory  via piecewise
constant interpolation as follows:
For
 $t   \in [0,\horizon]$,
define the continuous-time piecewise constant interpolated process parametrized by the step size
$\step$ as  
\beq  
\eth^\step(t) = \eth_\dtime , \; \text{ for } \ t\in [\step \dtime, \step \dtime+ \step). \label{eq:interpolatedp} \eeq

\subsection{Ordinary Differential Equation Limits of (\ref{eq:passive}) and (\ref{eq:mcmcirl})}   \label{sec:ODEinformal}

In this section we present  an informal averaging analysis which yields useful intuition  regarding the classical passive algorithm~(\ref{eq:passive}) and multi-kernel algorithm~(\ref{eq:mcmcirl}). Formal assumptions, theorem statements and  proofs are in Sec.\ref{sec:weak}.

\subsubsection{Classical Passive Gradient Algorithm} First consider the classical passive  gradient algorithm~(\ref{eq:passive}).
Suppose $\th_{\dtime}$ is  sampled i.i.d. from $\thdim$-variate
density $\belief(\cdot)$ and the noisy gradient $\D_\th \cost_{\dtime}(\th_{\dtime}) $ is available at each time $k$.
Assume that the noisy gradient $\D_\th \cost_k(\th_k)$ comprises additive noise:
$$\D_\th \cost_k(\th_k) = \nabla \Cost(\th_k) + \xi_k$$
where $\xi_k$ is a zero mean i.i.d. noise process
We  first {\it fix} the kernel step size  $\kernelstep$ and apply stochastic averaging theory arguments.
It indicates
that at the slow time scale, we can replace the fast variables (namely, $\xi_k$ and $\th_k$) by their expected value.
Then
the interpolated sequence $\th^\step(\cdot) $ converges weakly to the
ODE
\beq
\frac{d \eth}{dt}  =  h_1(\eth,\kernelstep) =  - \int_{\reals^\thdim}  \kerneln(\frac{ \th - \eth(t)}{\kernelstep}) \, \belief(\th)
\, \nabla
\Cost(\th) \, d\th. \label{eq:passiveODE}
\eeq
Finally, for sufficiently  small kernel step size $\kernelstep$,  the kernel  $\kerneln(\frac{ \th - \eth(t)}{\kernelstep}) $
behaves as Dirac delta function due to
(\ref{eq:kernel_properties}).
Therefore  as  $\kernelstep \downarrow 0$, the ODE
(\ref{eq:passiveODE})
becomes
\begin{equation}
  \label{eq:ODE1}
  \text{Classical Passive:}   \qquad    \frac{d \eth}{dt}  =   -\belief(\eth)
  \,\nabla \Cost(\eth).
\end{equation}
To make our  discussion of the multi-scale averaging more intuitive, we used two stepsizes $\e$ and $\kernelstep$.
In the averaging theory analysis, one can instead choose $\kernelstep$ to depend on $\e$. Then the two-step averaging is done simultaneously.

\subsubsection{Multi-kernel Algorithm}
Next, consider the multi-kernel passive algorithm~(\ref{eq:mcmcirl}) that is proposed in this paper.
Suppose $\{\th_{\dtime,l},l=1,\ldots,\numparticles\}$ are sampled i.i.d. from $\thdim$-variate density $\belief(\cdot)$ and the noisy gradients $\{\D_\th \cost_{\dtime,l}(\th_{\dtime,l}),l=1,\ldots,\numparticles\}$ are available at each time $k$.
   To give some insight, assume  that the  noisy gradient estimates have additive  measurement noise. So  for $\th \in \reals^\thdim$,
   \beq  \D_\th
\cost_{k,l}(\th) =   \nabla
\Cost(\th) + \xi_{k,l} \label{eq:nudef0} \eeq
where $\{\xi_{k,l}\}$ is a sequence of zero mean independent and identically distributed (i.i.d.) random variables. (In Sec.\ref{sec:weak} we will consider  more general mixing assumptions where the noise $\xi_{k,l}$ for each agent $l$  is correlated  over time $k$.)

First, from \eqref{eq:mcmcirl}, \eqref{eq:nudef0}, for fixed $\step$ and $\stepa$, as $\numparticles \rightarrow \infty$, it follows by self normalized importance sampling arguments that
\begin{equation}
  \label{eq:sirstep}
 \eth_{k+1} = \eth_k + \step \int_{\reals^\thdim} \nabla \Cost(\th) \, \pdf_\stepa(\th|\eth_k) \, d\th + \wdt\Noise^{\step}_k
\end{equation}
Here $\pdf_\stepa(\th|\eth_k) \propto \belief(\th) \,\pdf(\eth_k|\th)$ denotes the posterior conditional density of $\th$ given\footnote{We assume the existence of the conditional density $\pdf(\th|\eth)$.} $\eth_k$; recall $\pdf(\eth|\th) =\pdf_\stepa(\th-\eth)$ is specified in
\eqref{eq:cond-al}.
The noise variable in \eqref{eq:sirstep}, namely,
$$ \wdt\Noise^{\step}_k \defn \lim_{\numparticles\rightarrow \infty}\e \sum_{i=1}^\numparticles \weight_{k,i} \xi_{k,i} \rightarrow 0  \;\text{  as }\;
\step \rightarrow 0$$ if we choose  $\numparticles = o(1/\step)$; see formal proof
in Sec.\ref{sec:weak} and discussion point 5 below.
Second,   for sufficiently  small $\kernelstep$,  the posterior  density
$\pdf_\stepa(\th|\eth_\dtime)$   in \eqref{eq:sirstep} converges to a normal density.
Indeed, the Bernstein-von Mises theorem
\cite{Vaa00}
implies that for small parameter  $\kernelstep$ in the likelihood  (\ref{eq:cond-al}),
the posterior converges to the normal
density
  $\normal(\th;\eth_\dtime, \kernelstep^2 I_{\bar\th})$:  
\beql{p-app}\int |\pdf(\th|\eth_\dtime)- \normal(\th; \al_\dtime, \kernelstep^2 I_{\bar\th} )
| d\th \rightarrow 0
\text{ in probability under } P_{\bar \th}\eeq
Here
$I_{\bar \th}= \int_{\reals^\thdim} \nabla
\log \pdf(\eth| \th)\,
\pdf(\eth|\th) \, d\eth \vert_{\th = \bar \th}$
is  the  Fisher information matrix  evaluated at  ``true'' parameter value\footnote{It suffices to  choose any  $\bar \th$ such that
$\eth \sim \pdf(\cdot|\bar \th)$. The precise value of $\bar \th$ need not be known
 and is irrelevant to our analysis.}~$\bar \th$
and
\begin{multline} \label{Sig}
 \normal(\th; \al, \kernelstep^2 I_{\bar\th} )
\\ =
{2 \pi}^{-\thdim/2}
\exp\Big[ -\frac{1}{2} (\th-\eth)^\p   |\kernelstep^2 I_{\bar\th}^{-1}|^{-1}
(\th-\eth) \Big].
\end{multline}
Therefore, for small kernel step size $\stepa$,  \eqref{eq:sirstep} becomes
\begin{equation}
  \label{eq:sirstep2}
   \eth_{k+1} = \eth_k - \step \int_{\reals^\thdim} \nabla \Cost(\th) \, \normal(\th; \al_\dtime, \kernelstep^2 I_{\bar\th} ) \, d\th
 \end{equation}
 Next, as $\step\rightarrow 0$, stochastic averaging theory arguments imply that the interpolated sequence
 $\eth^\step(\cdot)$ defined in \eqref{eq:interpolatedp} generated by \eqref{eq:sirstep2} converges weakly to the ODE
\beq
\frac{d \eth}{dt}  =  h_2(\eth,\kernelstep) = - \int_{\reals^\thdim}   \nabla
\Cost(\th) \,\normal(\th; \al, \kernelstep^2 I_{\bar\th} ) \,   d\th \label{eq:multiODE}
\eeq
Finally, as   $\kernelstep \rightarrow 0$, $ \normal(\th; \al, \kernelstep^2 I_{\bar\th} )$ behaves  as a Dirac delta function  $\delta(\th-\eth)$;
so
 (\ref{eq:multiODE})
yields the limit ODE
 \begin{equation}
     \text{Multi-kernel  Passive:} \qquad    \frac{d \eth}{dt}  =  -\nabla \Cost(\eth) \label{eq:ODE2}
   \end{equation}

   \subsubsection{Discussion}
   To summarize, the continuous-time interpolated sequences from the passive algorithm
 (\ref{eq:passive}) and multi-kernel algorithm (\ref{eq:mcmcirl}) converge weakly to the ODEs
 (\ref{eq:ODE1}) and (\ref{eq:ODE2}), respectively.
 Note from (\ref{eq:ODE1}) that the ODE for the classical passive stochastic gradient algorithm depends on the sampling density
 $\belief(\cdot)$. In comparison  the ODE (\ref{eq:ODE2}) for the multi-kernel algorithm does not depend on $\belief(\cdot)$. Indeed,
 (\ref{eq:ODE2}) coincides with the ODE  of a standard stochastic gradient algorithm.

 Clearly both passive algorithms converge locally to a stationary point of $\Cost(\cdot)$. This is because  the set of stationary points of  $\Cost(\al)$, i.e., $\{\al^*: \nabla \Cost(\al^*)
 = 0\}$  are fixed points for both ODEs.

\subsubsection{Batch-wise Implementation of Passive Algorithm} \label{sec:batchwise}

In analogy to the multi-kernel algorithm \eqref{eq:mcmcirl}, one can  implement the classical  passive  algorithm \eqref{eq:passive}  on batches of length $\numparticles$ as
\begin{equation}
  \eth_{k+1} = \eth_k - \step  \frac{1}{\numparticles} \sum_{i=1}^\numparticles  \kerneln(\eth_{k} - \th_{k,i})\, \nabla_\th \cost_k(\th_{k,i}), \quad \th_k \sim \belief
  \label{eq:batchpassive}
\end{equation}
It can be shown using averaging theory arguments that algorithm
\eqref{eq:batchpassive}
has the same asymptotics as the classical passive algorithm \eqref{eq:passive}, namely  ODE \eqref{eq:ODE1} holds and also the asymptotic covariance is identical.
Furthermore, \eqref{eq:batchpassive} inherits the same problems with the scale factor $\kernelstep^{-\thdim}$ as \eqref{eq:passive}.
So there is no improvement with a batch-wise implementation compared to the classic passive algorithm  \eqref{eq:passive}.  Sec.\ref{sec:numerical}  compares  the performance of \eqref{eq:mcmcirl} with \eqref{eq:batchpassive} in numerical examples.

\subsubsection{Two-time Scale Interpretation}   The multi-kernel algorithm
\eqref{eq:mcmcirl} is a two-time scale algorithm. There are two approaches
for analyzing its behavior:
\\ {\em Approach 1. Asymptotic Scaling Limit.}
In the convergence analysis of  Sec.\ref{sec:weak}, we will
parametrize the batch size  $\numparticles$ by step size $\step$. Denoting  this as $\numparticles_\step$, we will  analyze the algorithm as
$\numparticles_\step \rightarrow \infty$ but  $\step \numparticles_\step \rightarrow 0$.
From a practical point of view,
for the convergence,
allowing  $\step \numparticles_\step \rightarrow 0$  means that the batch size  $\numparticles_\step$
can be chosen substantially smaller than the total data size of $O(1/\step)$.  For example,
we  can choose $\numparticles = o(1/\step)$, e.g., $\numparticles = \step^{1/q}$, for $q>2$. This analysis is, of course,  an idealization; but captures the essential scaling limit; and is widely used. The end result
is the ODE \eqref{eq:ODE2}.
\\ In Sec.\ref{sec:rate} we analyze the asymptotic covariance (rate of convergence) of the multi-kernel algorithm. In this analysis, we require  $\numparticles_\step = O(1/\step)$. The asymptotic covariance is  smaller than that of the classic batch-wise passive algorithm \eqref{eq:batchpassive} with  $\numparticles_\step = O(1/\step)$; see
discussion in Sec.\ref{sec:asympcov}
 below.
\\
{\em Approach 2. Finite $\numparticles$ analysis}.
An alternative  more messy  analysis involves  fixed  $\numparticles$, determining the approximation error, and then constructing the  limit.
Suppose $\sup_\th \|\Cost(\th)\|_\infty \leq \bar{\Cost}$ for some constant $\bar{\Cost}$. Then for finite $\numparticles$,
 Theorem 9.1.19 in \cite{CMR05} yields  the approximation error in
\eqref{eq:multiODE}
as:
\begin{multline}
   \E\{  \int_{\reals^\thdim}  \| \bigl( \pdf(\th|\eth(t)) - \hat{\pdf}_\numparticles(\th|\eth(t) ) \bigr) \nabla \Cost(\th) d\th \|^m \} \\ \leq \text{const} \,
\numparticles^{-m/2}\, \bar{\Cost}
\end{multline}
 Then the ODE
\eqref{eq:ODE2} has an additional bias term of  $O({\numparticles}^{-1/2})$ which affects
its fixed point.

In this paper we will deal with the asymptotic analysis
using  approach 1. This gives useful intuition as to why the algorithm works in terms of the asymptotic scaling limit.

 \subsection{Asymptotic Covariances} \label{sec:asympcov}

 For the classical passive algorithm, the dependence of the ODE  (\ref{eq:ODE1})  on  the sampling density $\belief(\cdot)$ affects the asymptotic rate
 of convergence; see \cite{YY96}. In Sec.\ref{sec:rate},
 we will study the rate of convergence of the multi-kernel algorithm
 (\ref{eq:mcmcirl}) with ODE (\ref{eq:ODE2}).
Also, in numerical examples discussed in Sec.\ref{sec:numerical},  we will show that the classical passive stochastic gradient algorithm
suffers from poor convergence rate for certain choices of $\belief(\cdot)$; whereas the multi-kernel algorithm does not.

Here we briefly give some intuition regarding the convergence rates of the passive and multi-kernel algorithms. In the stochastic approximation literature, the rate of convergence is specified in terms of scaling factor (related to the stepsize) together with
the asymptotic covariance of the estimates \cite{BMP90,KY03,Kri16}.
Assume for simplicity that the noise $\xi_k$ is i.i.d. with covariance $I$.
Let $\eth^*$ denote the fixed point of the ODE (\ref{eq:ODE2}).
Then assuming  $\nabla^2 \Cost(\eth^*)$ is positive definite,
the asymptotic covariance  $\asympCov$ of the
multi-kernel algorithm satisfies the algebraic  Liapunov equation (see Corollary
\ref{cor:asympcov} in Sec.\ref{sec:rate})
\begin{equation}
  \label{eq:L2}
\nabla^2 \Cost(\eth^*)  \,\asympCov + \asympCov\, \nabla^2\Cost(\eth^*) =    I
\end{equation}

In comparison, the rate of convergence for the classical passive algorithm (\ref{eq:passive}) and batch-wise implementation \eqref{eq:batchpassive} is slower; it
depends on the smoothness of the kernel similar to typical cases in nonlinear regression~\cite{YY96}. This, in fact, is a well known fact in nonparametric statistics.
Also as mentioned in Sec.\ref{sec:intro}, from an implementation point
of view, the scale factor $\kernelstep^{-\thdim}$ in classical passive algorithm is problematic since it blows up  for moderate to large $\thdim$; this requires using a very small step size in \eqref{eq:passive}.


\section{Weak Convergence Analysis  of Multi-kernel Passive Recursive Algorithm} \label{sec:weak}

This section is organized as follows. First we formally justify  (\ref{eq:sir}) as an un-normalized importance sampling estimator. Then  weak convergence of the multi-kernel algorithm to the ODE
(\ref{eq:ODE2}) is proved.

\subsection*{Additive Noise Assumption}
Recall   $
\D_\th c_{k,l}(\th)$  denotes the estimate of gradient $\nabla \Cost(\th)$. 
In  this section we define more explicit notation. We assume that the measurement noise in the gradient estimate is additive:
\beq  \D_\th
\cost_{k,l}(\th) =   \nabla
\Cost(\th) + \xi_{k,l} \label{eq:nudef} \eeq
Denote the sigma-algebra
${\mathcal G}_k = \sigma( \xi_{n,l}, n \le k)$.
Define
     \begin{equation}
       \label{eq:bxi}
     \bar \xi_k =
\E(\xi_{k,l}|{\mathcal G}_k), \quad
\Sigma_{\bar{\xi}_k}=
\cov(\xi_{k,l}|{\mathcal G}_k)
\end{equation}
We make the following assumptions regarding the cost $\Cost$, measurement noise $\xi_{k,l}$,  the sequence $\{\th_{k,l}\}$:

\begin{enumerate}[label=(A{\arabic*})]
  \item \label{A_cost} The function $C
(\cdot)$ has continuous partial derivatives up to the second order and the second partial derivatives are bounded uniformly.

\item \label{A_AbsCont}   The conditional density $p(\th|\al)$ exists.

\item \label{A_noise}
For each fixed $k$,  $\{\xi_{k,l}\}$  is i.i.d. over $l = 1,\ldots, \numparticles$ with  $\E (|\xi_{k,l}|^2 |{\mathcal G}_k)<\infty$.

\item \label{A_mix}
The sequence
$\{\bar\xi_k\}$ defined in \eqref{eq:bxi} is a stationary mixing process  with mixing measure $\varphi_j$ such that $\E |\bar \xi_k|^2 <\infty$ and
    $$
\sum_k \varphi^{1/2}_k < \infty.$$

\item \label{A_th}  The sequence $\{\th_{k,l}\}$ has independent rows and independent columns  sampled from the density $\belief$ such that
  for each fixed $k$, $\E \th_{k,l} =\bar \th$ and for each fixed $l$, $\E \th_{k,l}=\bar \th$. In addition, $\E |\th_{k,l}|^2 < \infty.$
\item \label{A_normal}
 $\int (1 + \|
 \nabla
 \Cost(\th)\|^2)\, \Big(\frac{\pdf(\th|\eth)}{\belief(\th)}\Big)^2 \, \belief(\th)\, d\th < \infty $
  \setcounter{assum_index}{\value{enumi}}
 \end{enumerate}

 \subsubsection*{Discussion of Assumptions}

 The additive noise assumption in \eqref{eq:nudef} together
   with \ref{A_noise} and \ref{A_mix} allows for
the general case where the gradient estimates are asymptotically unbiased (in $k$). In particular, noise $\xi_{k,l}$ can be correlated over time $k$ as long as it  satisfies stationary mixing conditions.
These are typically the minimal conditions required for establishing convergence of a stochastic gradient algorithm. 

 Regarding \ref{A_cost}, only first order differentiability is required  for the self-normalized
 importance sampling (Theorem \ref{thm:is}) and the ODE analysis (Theorem \ref{thm:conv1}). Second order differentiability is used in the rate of convergence
 (Theorem \ref{thm:rate}).

\ref{A_AbsCont} assumes the existence of the conditional density.
A sufficient condition is that  the conditional distribution
$\wdt P(\th|\al)$ is absolutely continuous w.r.t. the Lebesgue measure; this absolute continuity then implies  existence of conditional density $p(\th|\al)$.  Then
 $\E_{\pdf(\th|\eth)}
\D_\th \cost_{k,l}(\th)$ is well defined.

 \ref{A_noise}  facilitates modeling a  multi-agent system (such as a crowd sourcing example)  comprising $\numparticles$ independent agents, where the  pool of $\numparticles$ samples $\{
 \D_\th \cost_{k,l}(\th_{k,l})\}$ generated by the agents at each time $k$ have i.i.d. noise.
 But the parameters of the noise can be $k$ dependent, i.e.,  allowed to evolve with time.

\ref{A_mix} facilitates  modeling correlated measurement noise over time $k$. Essentially, a mixing process is one whose remote past and distant future are asymptotically independent; see \cite{EK86} for further details.
From a modeling point of view, this means that  the measurement noise of each sampling agent $l$  is correlated over time.
Of course, in the special case where $\xi_{k,l}$ is i.i.d. over $k,l$, then $\bar \xi_k = 0$
and  $\Sigma_{\bar{\xi}_k}$  is  constant independent of $k$.

\ref{A_th} models how the agent samples $\th_{k,l}$ to evaluate the noisy gradient $\D_\th \cost_{k,l}(\th_{k,l})$. We assume this sampling process is i.i.d.  As in \cite{YY96}, this can be generalized to correlated sampling from a Markov process with stationary distribution $\belief(\cdot)$. 

\ref{A_normal} is a
 classical square integrability assumption for asymptotic normality.


\subsection{Self-normalized Importance Sampling}
The aim here is  to  prove (\ref{eq:sir}) and also asymptotic normality of the estimate.
Recall in the main algorithm (\ref{eq:mcmcirl}) that $\th_{\dtime,l}$ is sampled from $\belief(\cdot)$.

 The
term   $\sum_{i=1}^\numparticles \weight_{\dtime,i}(\al_k) \,
\D_\th \cost_\dtime(\th_{\dtime,i}) $ in the
multi-kernel passive algorithm (\ref{eq:mcmcirl}) is a self-normalized  importance sampling estimator. Indeed, it can be obtained by the following argument:
\beq  \int
\D_\th \cost_{k,l}(\th)\, \pdf(\th|\eth_k)\, d\th=
\frac{\E_{\belief}\{
\D_\th \cost_{k,l}(\th)\,\pdf(\eth_k|\th)\}}
{\E_{\belief}\{ \pdf(\eth_k|\th)\}} \label{eq:is_motivation}
\eeq
Recalling the weights $\weight_{\dtime,l}$ defined in (\ref{eq:weights}),
the right-hand side
of (\ref{eq:is_motivation}) yields the implementation
 (\ref{eq:sir}). Below we prove that the estimate (\ref{eq:sir}) converges w.p.1 to (\ref{eq:is_motivation}).

Denote the estimated mean  and conditional expectation as
\begin{align*}
\hat{\mean}_{k,\numparticles}(\eth_k) &= \sum_{l=1}^\numparticles \weight_{\dtime,l}(\al_k) \,\D_\th\cost_{\dtime,l}(\th_{\dtime,l}) \\
\mean_k(\eth_k) &=  \E\{
\D_\th \cost_{k,l}(\th_{k,l}) | \eth_k \} = \int
\nabla \Cost(\th) \pdf(\th|\eth_k) d\th + \bar \xi_k
\end{align*}
where $\bar \xi_k$ is defined in \eqref{eq:bxi}.
 Note that $\hat{\mean}_{k,\numparticles}(\eth) $ is a self-normalized importance sampling estimate with proposal density $\belief(\th)$ and target density $\pdf(\th|\eth)$.



\begin{theorem}  \label{thm:is}
  \begin{compactenum}  \item  Assume  \ref{A_cost}-\ref{A_th} hold.
     Then
 $$\hat{\mean}_{k,\numparticles}(\eth) \rightarrow \mean_k(\eth) \text{  w.p.1
   as }  \numparticles \rightarrow \infty  $$
So in the special case $\xi_{k,l}$ is i.i.d. in $k,l$,     (\ref{eq:sir}) holds.
 \item  Assume  \ref{A_cost}-\ref{A_normal}.
 Then the following asymptotic normality holds:
\begin{equation}
 \sqrt{\numparticles} \big[ \hat{\mean}_{k,\numparticles}(\eth) - \mean_k(\eth)] \rightarrow
 \normald(0, \Sigma_k(\eth))
\end{equation}
 where (recall $\Sigma_{\bar\xi_k}$ is defined in \eqref{eq:bxi})
 \begin{multline} \Sigma_k(\eth) = \int \frac{\pdf(\th|\eth)}{\belief^2(\th)} \Bigl[ (
 \nabla\Cost(\th) - \mean(\eth) )   ( 
 \nabla \Cost(\th)
 - \mean(\eth) ) ^\p \\ + \Sigma_{\bar\xi_k} \Bigr] \belief(\th)\, d\th
   \label{eq:sigeth}
 \end{multline}
\end{compactenum}
\end{theorem}


\begin{proof}
For  notational convenience,   we divide the numerator and denominator of $\weight_{\dtime,i}$
 defined in (\ref{eq:weights})  by $\belief(\eth_\dtime)$. So
 $$   \weight_{\dtime,i}(\eth)  =   \frac{\pdf(\th_{\dtime,i}|\eth_k) / \belief(\th_{k,i}) }{\sum_{l=1}^\numparticles \pdf(\th_{\dtime,l}|\eth_k)/ \belief(\th_{k,l})}  , \quad i = 1,\ldots, \numparticles  $$

By  \ref{A_noise},  $\xi_{k,l}$ is i.i.d. in $l$ for fixed $k$. So by Kolmogorov's strong law of large numbers,
\begin{equation} \label{eq:consis}\begin{split}
&\numparticles^{-1} \sum_{l=1}^\numparticles
\D_\th \cost_{k,l} (\th_{k,l}) \,
\frac{\pdf(\th_{k,l}|\eth)}{\belief(\th_{k,l})} \rightarrow \mean_k(\eth) \; \text{ w.p.1}
\\
&\numparticles^{-1} \sum_{l=1}^\numparticles
\frac{\pdf(\th_{k,l}|\eth)}{\belief(\th_{k,l})} \rightarrow 1 \; \text{ w.p.1}
\end{split}
\end{equation}
Thus statement 1 holds.

To demonstrate the asymptotic normality, note that
\begin{multline*}
\sqrt{\numparticles} [ \hat{\mean}_{k,\numparticles}(\eth) - \mean_k(\eth)] = \\ \frac{\numparticles^{-1/2}
  \sum_{l=1}^\numparticles \frac{\pdf(\th_{k,l}|\eth)}{\belief(\th_{k,l})}  [
  \nabla \Cost(\th_{k,l}) + \xi_{k,l} - \mean_k(\eth)]}{\numparticles^{-1}  \sum_{l=1}^\numparticles \frac{\pdf(\th_{k,l}|\eth)}{\belief(\th_{k,l})} }
\end{multline*}
where we used (\ref{eq:nudef}) for $
\D_\th \cost_{k,l}(\th_{k,l}) $.
By the central limit theorem for i.i.d. random variables, under
 \ref{A_cost}-\ref{A_normal},
the numerator converges weakly to $\normald(0,\Sigma_k(\eth))$ with $\Sigma_k(\eth)$ defined in (\ref{eq:sigeth}).
Also by (\ref{eq:consis}) the denominator converges w.p.1 to 1.
Then by Slutsky's theorem, statement 2 holds.
\end{proof}

\subsection{Weak Convergence of Multi-kernel Algorithm to ODE}

Recall from  \eqref{eq:nudef} that the noise in the gradient
estimate is additive.
From \ref{A_noise},
 for each fixed $k$,
$\{\xi_{k,l}\}_l$ is an i.i.d.
sequence and from \ref{A_mix},  $\{\bar \xi_k\}$ is a sequence of
$\phi$-mixing noise.
The multi-kernel algorithm (\ref{eq:mcmcirl}) can be written as
\beql{sa-r} \al_{k+1}=\al_k- \e \sum^{\numparticles_\e}_{i=1} \gamma_{k,i} (\al_k) [\nabla C(\th_{k,i})+ \xi_{k,i}]),\eeq
where
\beql{gam-def}
\gamma_{k,i} (\al)= {\pdf(\eth|\th_{k,i}) \over \sum^{\numparticles_\e}_{l=1} \pdf(\eth|\th_{k,l})},\eeq
and $\numparticles_\e \to \infty$ as $\e\to 0$.
For simplicity, we assume that the initial iterate $\al_0$ is a constant independent of $\e$.
Rather than working with the discrete iteration, we consider a continuous-time interpolation. Define $\al^\e(t)= \al_k$ for $t\in [\e k, \e k +\e)$.
We proceed to analyze the convergence of the algorithm.
First, we specify the conditions needed for the convergence study.

\begin{enumerate}[label=(A{\arabic*})]
   \setcounter{enumi}{\value{assum_index}}

\item \label{A_ODE}  The following conditions hold.
\begin{itemize}
\item[(a)]
Conditions \ref{A_noise}-\ref{A_th} hold and
$\{\th_{k,i}\}$ and $\{\xi_{k,i}\}$ are independent.


\item[(b)]
 For $\rho (\th) =\th$ and $\rho(\th)=\D \cost_k(\th)$,
$\E | \rho (\th_{k,i})|^2 <\infty$.
In addition,
for each
$\al$, as $\e \to 0$, $\numparticles_\e\to\infty$, and
\beql{ave-c}\barray
 \disp\sum^{\numparticles_\e}_{i=1} \gamma_{i}(\al)
  \rho (\th_{k,i}) \ad\to
  \int_{\reals^\thdim} \rho(\th ) p(\th |\al) d \th \ \hbox{w.p.1.}\earray \eeq
\end{itemize}
\item \label{A_lik} (a)  The conditional probability density function
        \beql{eq:cond-al2}p(\eth|\th) =  p_v(\th - \al)\eeq
        where $ p_v(\cdot)$  is
a symmetric density with zero mean and covariance $O(\kernelstep^2) I$.
where $I$ denotes the identity matrix.  Moreover, $0<\kernelstep \rightarrow 0$, and as $\kernelstep\to 0$, \eqref{p-app} holds.
\\
(b) The Fisher information matrix $I_{\th}= \int_{\reals^\thdim} 
\nabla
\log \pdf(\eth| \th)\,
\pdf(\eth|\th) \, d\eth $ is invertible for all $\th
\in \reals^\thdim$.
\setcounter{assum_index}{\value{enumi}}
    \end{enumerate}

  {\em Remarks}.  (i) In \ref{A_ODE}(b), $\rho(\th)=\nabla C(\th)$ is used in the proof of the weak convergence theorem below, whereas $\rho(\th)=\th$ is used in the rates of convergence in Sec.\ref{sec:rate}.

Note that in \ref{A_ODE}(b), we assume that $\numparticles_\e \to \infty$ as $\e\to 0$. However, for the convergence part, we do not  restrict the way  it goes to $\infty$.
For the rates of convergence
result, we need to specify the rate of $L_\e$ goes to $\infty$;
see the specification in Theorem \ref{thm:rate}.

(ii)   \ref{A_lik} is used in the Bernstein von-Mises theorem to show that the posterior $\pdf(\th|\eth)$ is asymptotically normal and behaves as a Dirac delta as $\kernelstep \downarrow 0$; see Sec.\ref{sec:conv1}.

   \subsubsection*{Outline of Proof}
Let $ f\cd: \reals^\thdim\to  \reals$ such that $f \cd\in  C^1_0$ ($C^1$ function with compact support).
We define an operator $\L_1$ as follows:
\beql{op-l1}
\L_1 f(\al )= - f'_\al(\al) \Big[ \int_{\reals^\thdim} \nabla C(\th) \,p(\th | \al) d\th\Big]\eeq

For convenience, the proof proceeds in two steps. First Proposition \ref{prop:conv} shows  that
$\eth(t)$ satisfies the ODE (\ref{eq:multiODE})  w.r.t. the  conditional expectation $\pdf(\th|\eth(t))$.

\begin{proposition}\label{prop:conv}
Assume that assumptions {\rm \ref{A_cost}} and {\rm \ref{A_ODE}} hold
and that
equation \beql{ode} \dot \al(t)=-  \int_{\reals^\thdim} \nabla C(\th) p(\th| \al(t)) d\th\eeq
has a unique solution for each initial condition. Then the interpolated process
$\al^\e \cd$ converges weakly to $\al\cd$ such that
$\al\cd$ is the solution of \eqref{ode}. \qed
\end{proposition}

Next, Theorem \ref{thm:conv1} considers the limit of ODE (\ref{ode}) as the parameter $\kernelstep \rightarrow 0$ in the likelihood density $\pdf(\eth|\th)$. This yields the ODE
(\ref{eq:ODE2}) for the multi-kernel algorithm, and is our main result.

 \begin{theorem}\label{thm:conv1}
 Assume conditions {\rm\ref{A_cost}}-{\rm\ref{A_lik}} hold. Then $\al^\e\cd$ converges weakly to $\al\cd$ such that
 $\al\cd$ satisfies
 \beql{ode-r} \dot \al (t)= - \nabla C(\al(t)) . \eeq
 \qed
\end{theorem}

The proof of Proposition  \ref{prop:conv} and Theorem \ref{thm:conv1} are given in the following  two subsections. The important consequence of
Theorem  \ref{thm:conv1} is that the ODE is identical to that of a classical stochastic gradient algorithm.

\subsection{Proof of Proposition \ref{prop:conv}}
We shall use a truncation scheme. We show that the interpolated process of the truncated process is tight and then obtain the weak limit of the sequence. The argument is through the stochastic averaging using martingale methods.
Suppose that $M>0$ is fixed but otherwise arbitrary and $q_M(\al)= 1$ if $\th \in S_M=\{\al: |\al| \le M\}$, $q_M(\al)= 0$ if $\th \in  \reals^\thdim- S_{M+1}$, and $q_M(\cdot)$ is sufficiently smooth otherwise.
Because it is not known {\it a priori} that the sequence $\{\al_k\}$ is bounded, we define a truncated algorithm of the following form.
\beql{sa-r-tr} \al^M_{k+1}=\al^M_k- \e \sum^{\numparticles_\e}_{i=1} \gamma_{k,i} (\al^M_k)[\nabla C(\th_{k,i}) +\xi_{k,i}] q_M(\al^M_k).\eeq
We then define $\al^{\e,M}(t)= \al^M_k$ for $t\in [\e k, \e k+\e)$. We proceed to show that $\al^{\e,M}\cd$ converges weakly to $\al^M\cd$ first. Then by letting $M\to \infty$, we prove that the untruncated process $\al^\e\cd$ converges to $\al\cd$ with the desired limit.

\begin{lemma}\label{tight} Assume {\rm \ref{A_cost}} and {\ref{A_ODE}} hold. Then
$\al^{\e, M}\cd$ is tight in $D([0,\infty); \reals^\thdim)$, the space of functions that are right continuous and have left limits endowed with the Skorohod topology.
\end{lemma}

\para{Proof of Lemma \ref{tight}.}
Note that by \eqref{sa-r-tr} and the definition of interpolation,
we have
\beql{expan} \al^{\e,M} (t)= \al_0 - \e \sum^{(t/\e)-1}_{k=0} \sum^{\numparticles_\e}_{i=1} \gamma_{k,i} (\al^M_k)[\nabla C(\th_{k,i}) +\xi_{k,i}] q_M(\al^M_k),\eeq
By \ref{A_mix}, $\{\xi_{k,i}\}$ is uniformly integrable
and  $$-\e\sum^{(t/\e)-1}_{k=0}\sum^{\numparticles_\e}_{i=1} \gamma_{k,i} (\al^M_k)q_M(\al^M_k) \xi_{k,i}$$ is uniformly integrable.
The truncation, the continuity of $\nabla  C$,
 the definition of $\gamma_{k,i}(\al)$,
 and the moment bound conditions in \ref{A_ODE} implies that $\nabla C (\th_{k,i})$ and hence $$ - \e \sum^{(t/\e)-1}_{k=0} \sum^{\numparticles_\e}_{i=1} \gamma_{k,i} (\al^M_k)\nabla C(\th_{k,i})q_M(\al^M_k)$$ is also uniformly integrable. As a result, \cite[Lemma 3.7, p. 51]{Kus84} implies that $\{\al^{\e,M}\cd\}$ is tight as desired.
 \qed

Because $\{\al^{\e,M}\cd\}$ is tight,
it is sequentially compact. By Prohorov's theorem \cite{KY03},
there exists a weakly convergent subsequence. Denote this subsequence by $\{\al^{\e,M}\cd\}$ for notational simplicity and denote the limit as $\al^M\cd$.
By Skorohod representation \cite{KY03}, without changing notation
we may assume that $\al^{\e,M}\cd$ converges to $\al^M\cd$ w.p.1, and the convergence is uniform on any bounded time interval.

\begin{lemma}\label{weak-c-M} Under the conditions of Proposition \ref{prop:conv},
$\al^{\e,M}\cd$ converges weakly to $\al^M\cd$ such that
the limit is a solution of the martingale problem with operator $\L_1^M$.
\end{lemma}

\para{Proof of Lemma \ref{weak-c-M}.}
By virtue of the conditions of Proposition \ref{prop:conv}, the martingale problem with operator $\L^M_1$ has a unique solution (unique in the sense of in distribution). Note that
$\L^M_1$ has the same form as $\L_1$ but with $p(\th|\al)$ replaced by
\beql{pM-def} p^M(\th|\al)= p(\th|\al)\, q_M(\al).\eeq
Let $ f\cd: \reals^\thdim\to  \reals$ such that $f \cd\in  C^1_0$ ($C^1$ function with compact support).
We proceed to show that $\al^M\cd$ is a solution of the martingale problem with operator $\L^M_1$.
 For any $t,s>0$,
 partition the interval $[t/\e, (t+s)/\e)$ into subintervals of width $m_\e$ such that $m_\e\to \infty$ but
 \begin{equation}
   \label{eq:delta_e}
   \Delta_\e \defn \e m_\e\to 0, \quad \text{  as } \e\to 0
 \end{equation}
 It is readily seen that
\beql{f-exp}\barray\ad\!\!\!\!
f(\al^{\e,M} (t+s))- f(\al^{\e,M}(t) \\
\ad\! = \sum^{(t+s)/ \Delta_\e}_{l=t/ \Delta_\e}
[f(\al^M_{l m_\e + m_\e})- f(\al^M_{l m_\e})] \\
\ad\! = \e \!\!\sum^{(t+s)/\Delta_\e}_{l=t/ \Delta_\e} \!\! \Big[\psi^\e_l +\wdt \psi^\e_l +
e^\e_l \Big],
\earray\eeq
where
\beql{psi1}\barray\ad \psi^\e_l=-f'_\al(\al^M_{l m_\e})
\sumii \sum^{\numparticles_\e}_{i=1} \gamma_{k,i} (\al^M_k) \nabla C (\th_{k,i})q_M(\al^M_k)\\
\ad \wdt \psi^\e_l= -f'_\al(\al^M_{l m_\e})\sumii \sum^{\numparticles_\e}_{i=1} \gamma_{k,i} (\al^M_k) \xi_k q_M(\al^M_k),\earray\eeq
\beql{error-1}
\barray
e^\e_l \ad=
-[f'_\al(\al^{\e,M}({v^+}))- f'_\al(\al^M_{l m_\e})]\\
 \aad \quad  \times \sumii \sum^{\numparticles_\e}_{i=1} \gamma_{k,i} (\al^M_k) [\nabla C (\th_{k,i})  +\xi_{k,i}] q_M(\al^M_k).\earray\eeq
Above we used the notation  $ \sum_{k\in I _\e} =\sum^{lm_\e +m_\e -1}_{k=lm_\e}$,
$v^+$ is on the line segment joining $ \e lm_\e$ and $\e lm_\e + \e m_\e$,
$f_\al$ denotes the partial of $f$ w.r.t. $\al$, and $z'$ denotes the transpose of $z$.

Pick out any bounded and continuous function $h\cd$,
for each $t,s>0$, any positive integer $\kappa$, and any $t_\iota \le t$ with $\iota \le \kappa$,
we shall show that $\al^M\cd$ is the solution of a martingale problem with operator $\L_1$. To this end, it is readily seen that by the weak convergence and the Skorohod representation,
\beql{es-1}\barray\ad\lim_{\e\to 0}\E h(\al^{\e,M}(t_\iota): \iota \le \kappa)
[f(\al^{\e,M}(t+s)- f(\al^{\e,M}(t) ] \\
\aad =\E h(\al^M(t_\iota): \iota \le \kappa)[f(\al^M(t+s)- f(\al^M(t)].\earray\eeq
On the other hand, using \eqref{f-exp},
\beql{es-2}\barray\ad\!\!\!
\lim_{\e\to 0}\E h(\al^{\e,M}(t_\iota): \iota \le \kappa)
[f(\al^{\e,M}(t+s)- f(\al^{\e,M}(t) ]\\
\ad = \lim_{\e\to 0}\E h(\al^{\e,M}(t_\iota): \iota \le \kappa) \sum^{(t+s)/\Delta_\e}_{l=t/ \Delta_\e} [ \psi^\e_l+\wdt \psi^\e_l+ e^\e_l],\earray\eeq

Next, we work with the  term involving $\psi^\e_l$ in \eqref{es-2}.
Note that for any $k$ satisfying $ l m_\e \le k \le l m_\e +m_\e-1$, assuming $\e l m_\e\to v$ leads to $\e k \to v$.
Furthermore, $\al^{\e,M}(v)$ can be approximated by a ``finite valued process'' in that for any $\delta>0$, there is a $j_\delta$ so that we can choose $\{O^\delta_j: j\le j_\delta\}$ as a finite collection of disjoint sets with diameter $\le \delta$ and with the union of the sets covering the range of $\al^{\e,M}(v)$.
Denote by $\E_n$ the conditional expectation with respect to
${\cal F}_n$ the $\sigma$-algebra generated by $\{\th_{k,i}, \xi_{k,i}: k \le n\}$.
Using
\eqref{eq:cond-al} and the smoothness of $q_M\cd$, we have
\bea \ad\!\!\!\! \E_{lm_\e} \psi^\e_l\\
\ad\!\! = \!
{-\Delta_\e\over m_\e}f'_\al(\al^M_{l m_\e}) \E_{lm_\e}
\!\sumii\!\! \sum^{\numparticles_\e}_{i=1} \gamma_{k,i} (\al^M_{lm_\e}) \nabla C(\th_{k,i})q_M(\al^M_{lm_\e})\\
\aad \quad \hfill +\!o(1)\\
\ad \!\! =
{-\Delta_\e\over m_\e}f'_\al(\al^\delta_j) \E_{lm_\e}
\!\sumii\! \sum^{j_\delta}_{j=1}\sum^{\numparticles_\e}_{i=1} \gamma_{k,i} (\al^\delta_j) \nabla C(\th_{k,i})q_M(\al^\delta_j)\\
\aad \qquad\hfill\times 1_{\{\al^{\e,M}(v) \in O^\delta_j\}} +o(1)\\
\ad\!\!=
{-\Delta_\e\over m_\e}f'_\al(\al^\delta_j) \E_{lm_\e}
\!\sumii\! \sum^{j_\delta}_{j=1}[\E
 \nabla C (\th | \al^\delta_j)]\\
 \aad \qquad\hfill\times 1_{\{\al^{\e,M}(v) \in O^\delta_j\}} +o(1)
,\eea
where $o(1) \to 0$ in probability and $\E_{lm_\e}$ denotes the conditional expectation for the information up to $l m_\e$.
It then follows
\beql{es-4}\barray
\ad\!\!\!\! \lim_{\e\to 0}\E h(\al^{\e,M}(t_\iota): \iota \le \kappa)\Big[ \sum^{(t+s)/\Delta_\e}_{l=t/ \Delta_\e} \psi^\e_l\Big]\\
\ad\!\! = \E h(\al^{M}(t_\iota): \iota \le \kappa)\\
\aad \qquad \times \Big[-\!
\!\int^{t+s}_t \!\int_{\reals^\thdim}\! \nabla C(\th) p^M(\th| \al^M(v)) d\th dv\Big],
\earray\eeq
where $p^M(\th|\al)$ was defined in \eqref{pM-def}.

Likewise, we have
\bea \ad\!\!\!\! \E_{lm_\e} \wdt\psi^\e_l\\
\ad\!\! = \!
{-\Delta_\e\over m_\e}f'_\al(\al^M_{l m_\e}) \E_{lm_\e}
\!\sumii\!\! \sum^{\numparticles_\e}_{i=1} \gamma_{k,i} (\al^M_{lm_\e}) \xi_{k,i}q_M(\al^M_{lm_\e})\\
\aad \quad \hfill +\!o(1)\\
\ad \!\! =
{-\Delta_\e\over m_\e}f'_\al(\al^\delta_j) \E_{lm_\e}
\!\sumii\! \sum^{j_\delta}_{j=1}\sum^{\numparticles_\e}_{i=1} \gamma_{k,i} (\al^\delta_j) \xi_{k,i} q_M(\al^\delta_j)\\
\aad \qquad\hfill\times 1_{\{\al^{\e,M}(v) \in O^\delta_j\}} +o(1)\\
\ad\!\!=
{-\Delta_\e\over m_\e}f'_\al(\al^\delta_j)
\E_{lm_\e}
\!\sumii\! \sum^{j_\delta}_{j=1}\bar \xi_{k}
 \\
 \aad \qquad\hfill\times 1_{\{\al^{\e,M}(v) \in O^\delta_j\}} +o(1)
,\eea
where $o(1)\to 0$ in probability. Here the $o(1)$ comes from the finite value approximation of the $\al\cd$ process and the use of (A7).  
The mixing condition in \ref{A_mix} then implies
$${1\over m_\e} \sumii \bar\xi_k \to 0 \ \hbox{ w.p.1,}$$
because of $\bar \xi_k$ being stationary mixing implies that
it is strongly ergodic. Thus,
\beql{es-5}\barray
\ad\!\!\!\! \lim_{\e\to 0}\E h(\al^{\e,M}(t_\iota): \iota \le \kappa)\Big[ \sum^{(t+s)/\Delta_\e}_{l=t/ \Delta_\e} \wdt\psi^\e_l\Big] = 0.
\earray\eeq

The continuity of $f_\al$ and $v^+- \e l m_\e \to 0$ as $\e\to 0$ then yields
$$f'_\al(\al^{\e,M}({v^+}))- f'_\al(\al^M_{l m_\e})\to 0 \ \hbox{ as } \ \e\to 0$$
and as a result,
\beql{es-3}
\lim_{\e\to 0}\E h(\al^{\e,M}(t_\iota): \iota \le \kappa) \Big[\sum^{(t+s)/\Delta_\e}_{l=t/ \Delta_\e} e^\e_l\Big] = 0
.\eeq
Combining \eqref{f-exp}-\eqref{es-3},
$\al^M\cd$ is the solution of the martingale problem with operator $\L_1^M$.
The lemma is proved. \qed

\para{Completion of the Proof of Proposition \ref{prop:conv}.}
Next to complete the proof of Proposition \ref{prop:conv}, we will show that the untruncated process $\al^\e\cd$ converges to $\al\cd$.
The argument is similar to \cite[p. 46]{Kus84}. Thus we omit the details. The proof of the proposition is complete. \qed

\subsection{Proof of Theorem \ref{thm:conv1}} \label{sec:conv1}
Here we use the Bernstein von-Mises theorem to characterize the posterior as a normal distribution when the parameter $\kernelstep$ in the likelihood   density goes to zero.

Recalling \ref{A_lik}, by virtue of \eqref{p-app},  $p(\th| \al)$
can be approximated by
$ \normal(\th; \al, \kernelstep^2 I_{\bar\th} )$,
the normal density given by \eqref{Sig}.
For notational convenience denote $\wdt p(\th,\al) \defn  \normal(\th; \al, \kernelstep^2 I_{\bar\th} )$ below.
  Now, we work with $\kernelstep\to 0$.
By  Taylor expansion,
$$ \nabla C(\th) = \nabla C(\al ) + \nabla ^2 C (\al_+) [ \th -\al],$$ where  $\nabla^2 C$ is the Hessian (the second partial derivatives) of $C$, and $\al_+$ is on the line segment joining $\th$ and $\al$.
Recall that $v$ is chosen below \eqref{es-2}. That is, $\e l m_\e \to v$ as a result, for any $k$ satisfying $ lm_\e \le k \le lm_\e +m_\e$,
$\e k \to v$.
It follows that
\beql{nabC}\barray \ad \int_{\reals^\thdim} \nabla C(\th) p(\th | \al(v)) d\th
\\
\aad =  \int_{\reals^\thdim} \nabla C(\th)
\wdt p(\th, \al(v)) d\th + o_\kernelstep(1)
\\
\aad  =  \int_{\reals^\thdim}\nabla C(\al) \wdt p(\th, \al(v)) d\th\\
\aad \qquad + \int_{\reals^\thdim}\nabla^2 C(\al_+(v)) [\th -\al(v)]\wdt p(\th, \al(v)) d\th +o_\kernelstep(1) \\
\aad  = \nabla C(\al(v)) +o_\kernelstep(1)\\
\aad  \to \nabla C(\al(v)) \ \hbox{ as } \ \kernelstep \to 0,
\earray\eeq
where
$o_\kernelstep(1) \to 0$
$\kernelstep\to 0$. The form of the density implies that the integral of the term on the fourth line is zero.
Thus, we obtain the following result.

\begin{corollary}\label{cor:long}
Suppose that $\{\al^\e(t): t\ge 0, \e>0\}$ is tight
and there is a unique stationary point $\al^*$ of \eqref{ode-r}, which is stable in the sense of Liapunov. Then under the conditions of Theorem \ref{thm:conv1},
$\al^\e(\cdot + t_\e)$ converges weakly to $\al^*$ as $\e\to 0$, where $t_\e$ is any sequence satisfying $t_\e\to \infty$ as $\e\to 0$.
\end{corollary}

\para{Idea of Proof.} Since the idea is mainly from \cite{KY03}, we will only discuss the main aspects here. Choose $T>0$ and consider the pair of sequences $(\al^\e(t_\e+ \cdot), \al^\e(t_\e -T+\cdot))$ with limit $(\wdt\al\cd, \wdt\al_T\cd)$.
We have $\wdt\al(0)= \wdt\al_T(T)$. Although the value of
$\wdt\al_T(0)$ is  not known,  all the possible such $\wdt\al_T(0)$, over all $T$ and all convergent subsequences
belong to a  set that is tight. Then it can be shown that
for any $\delta>0$, there is a $T_\delta$ such that
for all $T\ge T_\delta$, $\wdt\al_T\cd$ will be in a neighborhood of $\al^*$ with probability $1-\delta$.
 This yields the desired conclusion, since it implies
 $\wdt \al(0)= \al^*$.
 This gives us the asymptotic properties for small $\e$ and large $t$.

{\em Remark}.
\label{rem:st}
For simplicity, we have assumed that the set $\{\al^\e(t): t\ge 0, \e>0\}$ is tight. This tightness can be verified, if we use perturbed Liapunov function methods \cite[Chapter 6 and 8]{KY03}
together with appropriate sufficient conditions, which we will not pursue here.

{\em Summary}. To get the result of Theorem \ref{thm:conv1}, we did the proof in two steps. The first step focused  on the case with $\kernelstep$ fixed, namely Proposition
\ref{prop:conv}.
The second step obtained the desired result by letting $\kernelstep\to 0$.
The result may also be obtained directly, if we let $\kernelstep=\kernelstep_\e$ such that $\kernelstep_\e\to 0$ as $\e\to 0$. We used the two-stage approach because it is easier to present the main ideas.

\section{Convergence Rate of Multi-kernel Algorithm}
\label{sec:rate}
In this section the rate
of convergence (diffusion approximation of estimation error)  of the multi-kernel passive algorithm is addressed.
Specifically we analyze the  dependence of  $\al^\e(t)-\al^*$ on $\e$.
The study is done through the analysis of the asymptotic distribution of a scaled sequence $(\al^\e(t)-\al^*)/\sqrt \e$.

\subsection{Main Results of 
  Diffusion Limit and Asymptotic Covariance}

Again for 
notational convenience we use
$$\wdt p(\th,\al) \defn  \normal(\th; \al, \kernelstep^2 I_{\bar\th} )$$
We need the following additional  assumption.

\begin{enumerate}[label=(A{\arabic*})]
  \setcounter{enumi}
  {\value{assum_index}}
\item 
\label{A_rate}
With $\kernelstep = \e$, there is a $d_0 >0$ such that
$(p(\th|\al) - \wdt p(\th,\al))/ \e^{(1/2)+d_0}$
 is bounded.
\end{enumerate}

Assumption \ref{A_rate} is based on \cite{HM76} where convergence rates are given for the
Bernstein von Mises theorem
(see also \cite[Theorem 3.1]{KV12a}, in which
 the authors examined the convergence rates in terms parameter estimates). Here we use conditions on the corresponding probabilities.

Starting with  algorithm \eqref{sa-r},
define $u_k= (\al_k-\al^*)/\sqrt \e$.
 Our main result regarding the rate of convergence of the multi-kernel passive algorithm (\ref{eq:mcmcirl}) is the following:

\begin{theorem} \label{thm:rate}  Assuming conditions of Corollary \ref{cor:long} and {\rm \ref{A_rate}} hold. Assume also that there is a $K_\e$ such that $\{ u_k: k\ge K_\e\}$ is tight. In addition, we choose $L_\e$ so that $L_\e = 1/\e$.
Define $u^\e(\cdot)$ as
\beql{uk-def}
u^\e (t)= u_k \ \hbox{ for } \ t\in [\e (k-K_\e), \e (k-K_\e)+ \e).\eeq
Then $u^\e\cd$ converges weakly to $u\cd$ such that $u\cd$ is a solution of the stochastic differential equation
\beql{sde-lim} d u = - \nabla^2 C(\al^*) \,u \,dt + \Sigma^{1/2} \,dw,\eeq
where $w\cd$ is a standard Brownian motion and $\Sigma$ is the covariance defined in \eqref{cov-x}. \qed
\end{theorem}

{\em Remarks}. (i) \eqref{sde-lim} is a functional central limit theorem; namely, the scaled interpolated error process $u^\e \cd$ of the multi-kernel algorithm converges to a Gaussian process $u\cd$.\\
(ii) In Theorem \ref{thm:rate},
we assumed that there is a $K_\e$ such that $\{u_k: k\ge K_\e\}$ is tight. This tightness can be proved by the method of perturbed Liapunov function methods. Here we simply assume it; see
 \cite[Chapter 10]{KY03} for further details.
In addition,
to obtain the desired limit, we  need to truncate $u_k$,  defined as $u_k^M$, and then consider the interpolated process $u^{\e,M}(t)$ similar to the proof of the convergence. However, for notational simplicity, we will not use the truncation device, but proceed as if the iterates were bounded.

The main consequence of Theorem \ref{thm:rate} is that  we can determine the asymptotic covariance of the  diffusion \eqref{sde-lim}. In the stochastic approximation literature~\cite{KY03,BMP90},  this asymptotic covariance specifies the asymptotic rate of convergence. We have the following result for the multi-kernel algorithm.

\begin{corollary} \label{cor:asympcov}
Assume  $\nabla^2 C(\eth^*)$ is positive definite.  Then the diffusion~\eqref{sde-lim} is a stationary process with asymptotic marginal distribution
$\normald(0,\asympCov)$ where covariance  $\asympCov$ satisfies the algebraic Liapunov equation
\begin{equation}
  \label{eq:liap}
\nabla^2 \Cost(\eth^*)  \,\asympCov + \asympCov\, \nabla^2\Cost(\eth^*) =    \Sigma
\end{equation}
where $\Sigma$ is defined in \eqref{cov-x}.
\end{corollary}

To summarize, the solution $\asympCov$ of the algebraic Liapunov equation
(\ref{eq:liap}) is the asymptotic covariance (rate of convergence)
of the multi-kernel algorithm \eqref{eq:mcmcirl}.

\subsection{Proof of Theorem \ref{thm:rate}}
We have
\beql{rr0}\barray
 u_{k+1} =u_k\ad - \e
\sum^{\numparticles_\e}_{i=1} \gamma_{k,i} (\al_k) \nabla^2 C(\al^*) u_k   -  \sqrt \e  \bar \xi_{k} \\
\ad  -  \e
\sum^{\numparticles_\e}_{i=1} \gamma_{k,i} (\al_k) \nabla^2 C(\al^*) {\th_{k,i}- \al^* \over \sqrt \e}\\
\ad -\sqrt \e \sum^{\numparticles_\e}_{i=1} \gamma_{k,i} (\al_k) [\xi_{k,i}- \bar \xi_k] + \wdt g_{k,i} ,\earray
 \eeq
 where
\beql{gki} \barray \wdt g_{k}=\ad  - \e
\sum^{\numparticles_\e}_{i=1} \gamma_{k,i} (\al_k) [\nabla^2 C(\al_k^+)- \nabla^2 C(\al^*)] u_k \\
\ad   -  \e
\sum^{\numparticles_\e}_{i=1} \gamma_{k,i} (\al_k) [\nabla^2 C(\al_k^+)-\nabla^2 C(\al^*)] {\th_{k,i}- \al_k \over \sqrt \e}, \earray\eeq
and $\al^+_k$ is on the line segment joining $\theta_{k,i}$ and $\al^*$.

\begin{lemma}\label{lem:xi-to-bm}
Define $$w^\e(t) = -\sqrt \e \sum^{t/\e}_{k=0} \bar \xi_k,$$
where $t/\e$ is again understood to be the integer part of $t/\e$.
Then under {\rm \ref{A_cost}} and {\rm\ref{A_ODE}},
$w^\e\cd$ converges to a Brownian motion $\wdt w\cd$ whose covariance is $\Sigma t$ with $\Sigma$ given by
\beql{cov-x} \Sigma = \E \bar \xi_0 \bar \xi_0' +
 \sum^\infty_{k=1}\E \bar \xi_0 \bar \xi'_k + \sum^\infty_{k=1} \E \bar \xi_k \bar \xi'_0.
\eeq
\end{lemma}

\para{Proof.} The proof is well known and
can be found, for example, in \cite[Chapter 7, p.350]{EK86}. \qed

Consider \eqref{rr0}. We note that the terms
$\sum^{\numparticles_\e}_{i=1} \gamma_{k,i} (\al_k) \check H_{k,i}$,
with $\check H_{k,i}$ denoting each of the functions involved in the second, the fourth,  the fifth terms,  and in  $\wdt g_{k,i}$.
Then it is readily seen that
$\gamma_{k,i}(\al_k)$ can be replaced by $\gamma_{k,i}(\al^*)$
by the continuity of $\gamma_{k,i}\cd$, the tightness of $(\al_k-\al^*)/\sqrt \e$ for $k \ge K_\e$, and the tightness of $(\th_{k,i}- \th)/\sqrt \e$. Thus we can rewrite \eqref{rr0} as
\beql{rr0a}\barray
 u_{k+1} =u_k\ad - \e
\sum^{\numparticles_\e}_{i=1} \gamma_{k,i} (\al^*) \nabla^2 C(\al^*) u_k   -  \sqrt \e  \bar \xi_{k} \\
\ad  -  \e
\sum^{\numparticles_\e}_{i=1} \gamma_{k,i} (\al^*) \nabla^2 C(\al^*) {\th_{k,i}- \al^* \over \sqrt \e}\\
\ad -\sqrt \e \sum^{\numparticles_\e}_{i=1} \gamma_{k,i} (\al^*) [\xi_{k,i}- \bar \xi_k] + \wdh g_{k} +e_k ,\earray
 \eeq
 where $\wdh g_{k}$ is as $\wdt g_{k}$ but with $\al_k$ replaced by $\al^*$, and $ \sum^{(t+s)/\e}_{k=t/\e} e_k =o(1)\to 0$ in probability as $\e\to 0$.

 To proceed, we show that the effective terms for consideration of the desired limit is only from line 1 of \eqref{rr0a}.
 Let $f\cd \in C^2_0$ ($C^2$ functions with compact support).
 As in the proof of convergence of the algorithm,
 for any $t,s>0$, and $\kappa \in {\mathbb Z}_+$,
 and $t_\iota \le t$ with $\iota \le \kappa$,
 and any bounded and continuous function $h$, we work with $f(u^\e(t+s)) - f(u^\e(t))$ similar to the convergence proof.

 Recalling the definition of $\Delta_\e$ in \eqref{eq:delta_e},
it can be seen that
 \beql{ee1}\barray \ad\!\!\!\!
 \sqrt \e  \E h( u^\e(t_\iota): \iota \le \kappa)
 \sum^{(t+s)/\Delta_\e}_{l=t/\Delta_\e} f'_u(u_{l m_\e})\\
 \aad \qquad \qquad \times \sumii
 \sum^{\numparticles_\e}_{i=1} \gamma_{k,i} (\al^*) [\xi_{k,i}- \bar \xi_k]\\
 \ad  = \sqrt \e  \E h( u^\e(t_\iota): \iota \le \kappa)
 \sum^{(t+s)/\Delta_\e}_{l=t/\Delta_\e} f'_u(u_{l m_\e})\\
 \aad \qquad \qquad\times \sumii
 \sum^{\numparticles_\e}_{i=1} \gamma_{k,i} (\al^*) \E_{lm_\e}[\xi_{k,i}- \bar \xi_k]\\
 \ad = 0  \hbox{ because $\xi_{k,i}$ is independent of $\th_{k,i}$ and } \E_{{\cal G}_k} \xi_{k,i}= \bar \xi_k.
 \earray\eeq

 Next we work with the  term on the second line of \eqref{rr0a}:
 \beql{ee2}\barray \ad\!\!\!\!
\sum^{\numparticles_\e}_{i=1} \gamma_{k,i} (\al^*) \nabla^2 C(\al^*) {\th_{k,i}- \al^* \over \sqrt \e}\\
\ad ={1\over \sqrt \e}
\nabla^2 C(\al^*) \Big[\sum^{\numparticles_\e}_{i=1} \gamma_{k,i} (\al^*) \th_{k,i}
- \int_{\reals^\thdim}  \th  p(\th| \al^*) d\th \Big] \\
\aad \ +\nabla^2 C(\al^*)
{ \disp \int_{\reals^\thdim} [\th -  \al^* ] p(\th| \al^*)d\th \over \sqrt \e}
\earray\eeq

Note that
$\numparticles_\e =1/\e$.  Moreover, we can
make $(\th_{k,i}-\th)/\sqrt \e$ be bounded in probability. Thus
\bea \ad \Phi^\e_k = {1\over \sqrt \e}
\nabla^2 C(\al^*) \Big[\sum^{\numparticles_\e}_{i=1} \gamma_{k,i} (\al^*) \th_{k,i}
- \int_{\reals^\thdim}  \th  p(\th| \al^*) d\th \Big]\\
\aad = {1\over \sqrt {\numparticles_\e\e}}
\nabla^2 C(\al^*) \\
 \aad \qquad \qquad \times\sqrt {\numparticles_\e}\Big[\sum^{\numparticles_\e}_{i=1} \gamma_{k,i} (\al^*) \th_{k,i}
- \int_{\reals^\thdim}  \th  p(\th| \al^*) d\th \Big].\eea
As a result,
 \bea \ad \E h( u^\e(t_\iota): \iota \le \kappa)\Big[ \e \sum^{(t+s)/\Delta_\e}_{l=t/\Delta_\e} f'_u(u_{lm_\e}) \sumii  \Phi^\e_k \Big]\\
 \aad \to 0 \ \hbox { as } \ \e \to  0.\eea

 Note also
 \begin{multline*}
{ \disp \int_{\reals^\thdim} [\th -  \al^* ] p(\th| \al^*)d\th \over \sqrt \e}
 = { \disp \int_{\reals^\thdim} [\th -  \al^* ]\wdt p(\th, \al^*)d\th \over \sqrt \e} \\
 + { \disp \int_{\reals^\thdim} [\th -  \al^* ](p(\th|\al^*)-\wdt p(\th, \al^*))d\th \over \sqrt \e} .
\end{multline*}
The term on the second line above goes to zero as \eqref{nabC}.
As for the last line in the above,
using \ref{A_rate}, we have
 {$${p(\th|\al^*)-\wdt p(\th, \al^*) \over \sqrt \e}\to 0
 \ \hbox{  as } \ \e \to 0.$$}
 It then follows that
\bea \ad \E h( u^\e(t_\iota): \iota \le \kappa) \sum^{(t+s)/\Delta_\e}_{l=t/\Delta_\e}\Delta_\e f'_u(u_{lm_\e})\\
 \aad \quad \times {1\over m_\e} \sumii
 {p(\th|\al^*)-\wdt p(\th, \al^*) \over \sqrt \e}
 \to 0\eea
 as $\e\to 0$.
Likewise, detailed estimates as above implies that
\beql{eeee}
\E h( u^\e(t_\iota): \iota \le \kappa) \sum^{(t+s)/\Delta_\e}_{l=t/\Delta_\e} f'_u(u_{lm_\e}) \sumii \wdh g_k
\to 0 \ \hbox{ as } \ \e\to 0.\eeq

Thus, we can show that
\beql{eee3}\barray\ad
 f(u^\e(t+s))- f(u^\e(t))\\
 \aad = \sum^{(t+s)/\e}_{l=t/\Delta_\e}
f'_u (u_{lm_\e}) \Big[ \Delta_\e {1\over m_\e}  \sumii
\sum^{\numparticles_\e}_{i=1} \gamma_{k,i} (\al^*) \nabla^2 C(\al^*) u_k \\
  \aad \qquad \hfill - \sqrt \e \sumii   \bar \xi_{k} \Big] \\
\aad \quad + {1\over 2}\sum^{(t+s)/\e}_{k=t\e} \tr[
f_{uu} (u_{lm_\e}) \Delta_\e {1\over m_\e} \sumii \sum_{l\ge k} \xi_k\xi'_l]+ o(1),
\earray \eeq
where $o(1)\to 0$ in probability uniformly in $t$.
Using the same techniques as presented above, we
that $u^\e\cd$ converges weakly to $u\cd$ such that $u\cd$ is a solution of the martingale problem with operator
$$\op f(u) = {1\over 2} \tr [ f_{uu} (u) \Sigma] - f'_u(u) \nabla^2 C(\al^*) u.$$
Note also that in view of Lemma \ref{lem:xi-to-bm}, we can replace the Brownian motion $\wdt w\cd$ by $\Sigma^{1/2} w\cd$, where $w\cd$ is a standard Brownian motion.
Thus we have proved Theorem  \ref{thm:rate}.

\section{Example: Mis-specified Stochastic Gradient}
\label{sec:mis}

So  far we  considered the case where  the passive algorithm obtains  estimates $\D_\th \cost_\dtime(\th_\dtime)$
 at randomly chosen points
independent of its  estimate $\eth_\dtime$. That is, the  passive
algorithm  has no role in determining where the gradients are evaluated.

We now consider a modification where  the gradient algorithm receives a noisy version of the gradient evaluated
 at a stochastically perturbed value of $\eth_\dtime$.
 The setup comprises two entities: a stochastic gradient algorithm and
 an agent. The
  stochastic gradient algorithm requests a gradient to be evaluated at  $\eth_k$.
  The  agent  can only partially comprehend this request; each agent $l$ understands the request as   $$\th_{k}  = \eth_k + \snoise_{k},
  \quad \snoise_{k} \sim \pdf_\snoise(\cdot) \;\; \text{i.i.d.}$$
  The  agent evaluates its gradient $\D_\th \cost_k(\th_{k})$.
Finally the agent  sends
$\{\th_{k}, \D_\th \cost_k(\th_{k})$
 to the passive node.
This procedure repeats for $k=1,2,\ldots$.
So the gradient algorithm actively specifies where to evaluate the gradient; however, the  agent evaluates a noisy gradient  and that too at  a stochastically perturbed (mis-specified) point~$\th_k$.

Consider  the classical passive gradient algorithm
(\ref{eq:passive}).
By averaging theory, for fixed kernel step size $\kernelstep$,
the ODE is
\[ \frac{d\eth}{dt}  = - \int_{\reals^\thdim} \kerneln(\frac{ \th - \eth
  }{\kernelstep}) \, \nabla \Cost(\th)\, \pdf_\snoise(\th-\eth) d\th \]
Then as the kernel step size $\kernelstep \downarrow 0$, the ODE becomes
\[ \frac{d\eth}{dt}  =  - \pdf_\snoise(0)\,\nabla \Cost(\eth) \]
On the time scale $\tau = \pdf_\snoise(0) t$, this coincides with the ODE (\ref{eq:ODE2}) for the multi-kernel algorithm (\ref{eq:mcmcirl}).


Several applications motivate  the above framework.
One motivation is inertia. If the agent has dynamics, it may not be possible to  abruptly  jump to evaluate a gradient at $\eth_k$, at best the agent can only evaluate a gradient at a point $\eth_k + \snoise_k$. A second motivation stems from mis-specification:  if the
stochastic gradient algorithm represents a machine (robot) learning from  the responses of humans, it is difficult to specify to the human exactly what choice of
$\eth_\dtime$  to   use. Then $\th_k = \eth_k + \snoise_k$ can be viewed as an approximation to this  mis-specification. A third motivation stems from noisy communication channels: suppose the gradient algorithm transmits its request
$\eth_k$ via a noisy (capacity constrained) uplink communication channel, whereas the agent
(base-station) transmits the noisy gradients without transmission error in the downlink channel.




\section{Numerical Example. Passive LMS for Transfer Learning} \label{sec:numerical}
Transfer learning \cite{PY09} refers to using knowledge gained in one domain to learn in another domain.  We consider here an example of the constrained least mean squares (LMS) algorithm involving transfer learning; the LMS algorithm is arguably the most
widely used  stochastic gradient algorithm in signal processing and system identification.

To minimize the cost $C(\th) = \E\{ |\obs_k - \psi_k^\p \th|^2\} $ w.r.t.~$\th$, the LMS algorithm is
$$  \th_{\dtime+1} = \th_\dtime + \step  \, \psi_\dtime \, (\obs_\dtime -  \psi_\dtime^\p \th_\dtime) $$

By passively observing the sequence of gradient estimates
$\{\psi_\dtime \, (\obs_\dtime -  \psi_\dtime^\p \th_\dtime) \}$,
where $\th_k$ are sampled from density $\belief(\cdot)$,
 our aim is to ``transfer'' these passive observations to solve the linearly constrained stochastic optimization problem: Estimate
 \begin{equation}
   \label{eq:clms}
   \th^* = \argmin C(\th) \;\text{ subject to } \clhs^\p \th = \crhs
 \end{equation}
 {\em without} additional training data.
 Here $\clhs \in \reals^\thdim$ and $\crhs$ are assumed known. We assume that there are several such pairs $(\clhs,\crhs)$ for which we need to solve \eqref{eq:clms} simultaneously
 without additional training data. Therefore the classical constrained LMS algorithm does not work.

The linearly constrained adaptive filtering problem \eqref{eq:clms}  arises in several applications such as antenna array processing, adaptive beamforming, spectral analysis and system identification \cite{Van04,Fro72,GC86,YSL13}. The constraint is constructed from prior knowledge such as directions of arrival in antenna array processing; also linear equality constraints  are  imposed to improve robustness of the estimates.

Rewriting the above problem in the  two-time scale setting of \eqref{eq:mcmcirl},
we are in a passive setting with $\th_{k,l}$ are sampled from density $\belief(\cdot)$ and  the
gradient estimates
\begin{equation}
  \D_\th \cost_k(\th_{k,l})
  = -  \psi_{\dtime.l} \, (\obs_{\dtime,l} -  \psi_{\dtime,l}^\p \th_{\dtime,l}),
  \quad l=1,\ldots,\numparticles  \label{eq:lmsgrad}\end{equation}
are available to a passive observer.  Given $\{\th_{k,l}, \nabla_\th \cost_k(\th_{k,l})\}$, the observer wishes to {\em transfer} this information to  estimate the constrained  minimum $\th^*$ defined in \eqref{eq:clms}.
With  $\lagrange \in \reals$ denoting a fixed Lagrange multiplier, the {\em passive constrained LMS algorithm}
corresponds to  (\ref{eq:passive}) and (\ref{eq:mcmcirl}) with
gradients specified as
\begin{equation}
  \label{eq:lmspassive_grad}
  -  \psi_{\dtime.l} \, (\obs_{\dtime,l} -  \psi_{\dtime,l}^\p \th_{\dtime,l})
+ \lagrange \clhs,   \quad l=1,\ldots,\numparticles
\end{equation}

We now illustrate this passive LMS algorithm via numerical examples.
We chose  the true parameter is $\th^o = [1,\ldots,\thdim]^\p$ and the  observations are generated as
$\obs_\dtime = \psi_\dtime^\p \th^o + w_k$ where $\psi_k\sim \normald(0,I)$ and $w_k \sim \normald(0,I)$ are i.i.d. We chose $\lambda = 1$ and $\clhs = \ones_{\thdim}$.

For the multi-kernel algorithm (\ref{eq:mcmcirl}) we chose $\numparticles = 1000$ (recall $\numparticles$ is the fast time scale horizon),  and $\pdf_\obsnoise$ as a Laplace density (\ref{eq:laplace}). We found empirically that the Laplace kernel performed significantly better than the  Gaussian kernel.

\subsection{Comparison with Classical Passive LMS} \label{sec:static}
Here  we compare our proposed multi-kernel passive algorithm
\eqref{eq:mcmcirl} with the batch-wise implementation of the classical passive algorithm \eqref{eq:batchpassive}.
For  true parameter $\thtrue = [1, 2, 3, 4, 5]^\p$ is  easily verified
that the $\th^* = [2,3,4,5,6]^\p$. Both algorithms were run
with $\kernelstep = 0.2$, $\numparticles=1000$. The step size of the multi-kernel algorithm was fixed at $\step = 5 \times 10^{-4}$. For the classical algorithm we experimented with  various step sizes in order to obtain the best results.

First we consider the case when the sampling density $\belief(\cdot)$ is the density for the normal distribution $\normald(0,\sigmagauss^2 I_{\thdim})$
where the variance $\sigmagauss^2$ is specified below.
Table \ref{table:sim} displays the simulation results averaged over 100 independent trials. It can be seen that the multi-kernel algorithm yields substantially more accurate estimates (even though no tuning of the step size was done) compared to the classical passive algorithm (with optimized step size).  Also
for
$\sigmagauss > 15$, the classical passive  algorithm diverged and it was not possible  to obtain statistically reliable  estimates of  the parameters (as reflected  by the standard deviations in Table~\ref{table:sim}).
Moreover,
in  numerical results not presented here,
we found that the classical passive  algorithm \eqref{eq:passive} yielded identical results to the batch implementation algorithm \eqref{eq:batchpassive}.

 \begin{table}[h] \centering
   \begin{subtable}{0.45\textwidth}
     \centering
\begin{tabular}[t]{|c|c|c|} \hline
    \multirow{2}{*}{$\sigmagauss$} & RMSE (std)  & RMSE  (std) \\
      &  classical & multi-kernel \\
    \hline
    5  &  0.4555 (0.1441) & 0.5165  (0.0364)  \\
    10  &  0.4461 (0.1815) & 0.3073 (0.0667) \\
    15 & 1.3404 (1.5988) & 0.3445 (0.0975) \\
    20 &  4.8171 (6.9237) & 0.3737 (0.1258) \\
    25 & 11.0636 (21.0627)  & 0.4741 (0.1595) \\
    30 &  10.6585 (22.6474) & 0.5228 (0.1727) \\
                              \hline
\end{tabular}
\caption{Normal sampling density
  $\belief(\cdot)$ being the density of the normal distribution  $\normald(0,\sigmagauss^2 I_{\thdim})$.}
\label{table:sim}
\end{subtable}
\\ \vspace{0.4cm}
\begin{subtable}{0.45\textwidth} \centering
 \begin{tabular}{|c|c|c|} \hline
    \multirow{2}{*}{$\scal$} & RMSE (std)  & RMSE  (std) \\
      &  classical & multi-kernel \\
    \hline
    5  &  0.4789 (0.1904) & 0.4207  (0.0684)  \\
    10  &  1.8595  (1.5636) & 0.4078 (0.1085) \\
    15 &  5.6482  (3.7195) & 0.4602 (0.1423) \\
    20 &  -   & 0.6209 (0.2048) \\
    25 &  -  & 0.7161 (0.2136) \\
    30 & -  & 0.8413 (0.2530 \\
                              \hline
  \end{tabular}
\caption{Logistic sampling density \eqref{eq:logistic} with scale parameter $\scal$.}  \label{table:sim2}
\end{subtable}
\caption{Comparison of classical passive algorithm \eqref{eq:batchpassive} with multi-kernel algorithm
    \eqref{eq:mcmcirl} for $\numparticles=1000$ for normal and logistic sampling densities.
The RMSE error is $\|\th_k - \th^*\|_2$ at time $k=10^4$ averaged over 100 independent trials. The standard deviations (std) over these 100 independent trials are indicated in parenthesis.  The `$-$' represent unstable (statistically unreliable) estimates.}  
\end{table}

Next we illustrate the performance  when the sampling density $\belief(\cdot)$ is heavy-tailed. We simulated  each element  $\th(i)$, $i=1,\ldots,\thdim$ independently  from the univariate logistic density
with scale parameter $\scal$; so the sampling density is
\begin{equation}
  \label{eq:logistic}
  \belief(\th) = \prod_{i=1}^\thdim\frac{ \exp(-\th(i)/\scal)}{\scal(1+ \exp(-\th(i)/\scal))^2}, \quad \scal > 0
\end{equation}
Table \ref{table:sim2} compares the performance of the multi-kernel algorithm \eqref{eq:mcmcirl} with the classical passive algorithm \eqref{eq:batchpassive} for various values of $\scal$ (recall the variance of the logistic distribution is proportional to $\scal^2$). It is seen
from Table \ref{table:sim2} that the multi-kernel algorithm yields  significantly more accurate estimates.


\subsection{Non-stationary Transfer Learning. Tracking Behavior}

The fixed step size in the passive stochastic gradient algorithms facilitates tracking time varying parameters of a non-stationary stochastic  optimization problem. To
 illustrate the tracking performance of  algorithms (\ref{eq:passive}) and (\ref{eq:mcmcirl}), we now consider a non-stationary transfer learning problem where $\th^*$ jump changes, and this jump time unknown to the algorithm.
We chose the sampling density $\belief(\cdot)$  to be the density of $\normald(0,\sigmagauss^2 I_{\thdim})$
where the variance $\sigmagauss^2$  is specified below, for  choices of parameter dimensions $\thdim = 3, 5$.

To make a fair comparison, we ran
the classical passive algorithm (\ref{eq:passive}) for $1000 $ times the number of iterations of the multi-kernel algorithm \eqref{eq:mcmcirl}. The classical passive algorithm step size is $\step = 0.05$ (this gave the best response in our simulations) and Laplace kernel with $\kernelstep = 0.2$.

Figures~\ref{sim1}, \ref{sim2} and \ref{sim3}
illustrate  sample paths of the estimates of the multi-kernel and classical passive algorithm.
Also, as in Sec.\ref{sec:static},   we found that
the classical passive algorithm is highly sensitive to the sampling density $\belief(\cdot)$  compared to the multi-kernel passive algorithm. For example when the variance of $\belief(\cdot)$ is high, the classical passive algorithm suffers from poor convergence
(Figures~\ref{sim1}, \ref{sim2}). For small variance of $\belief(\cdot)$ the classical passive algorithm performs similarly to the multi-kernel algorithm (Figure~\ref{sim3}).

All the simulation results presented are fully reproducible with Matlab code presented in the appendix.

 \begin{figure}
   \centering
     \begin{subfigure}{.45\textwidth}
       \includegraphics[scale=0.45]{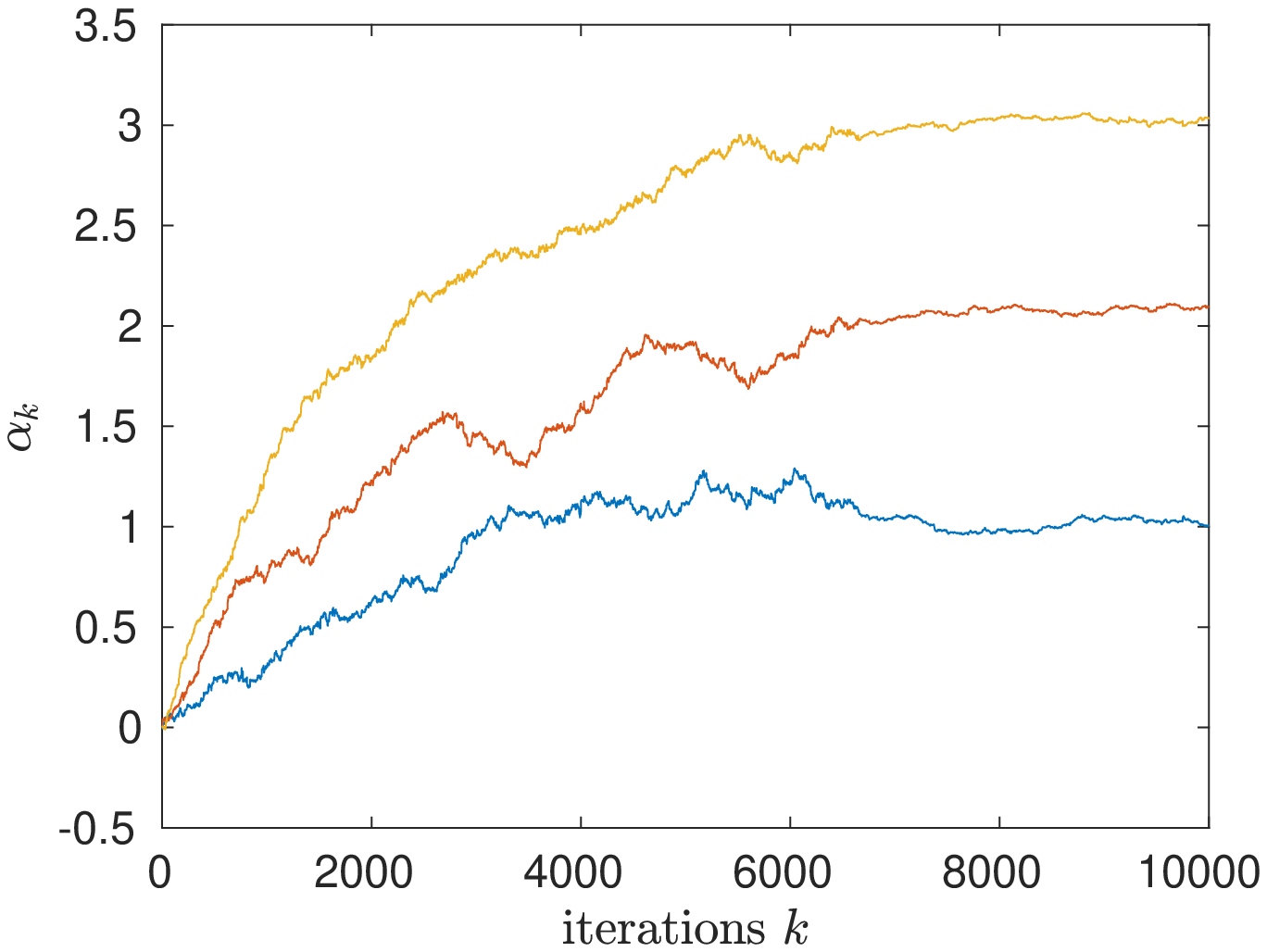}
        \caption{Multi-kernel Passive LMS}
           \end{subfigure}
 \begin{subfigure}{.45\textwidth}
   \includegraphics[scale=0.45]{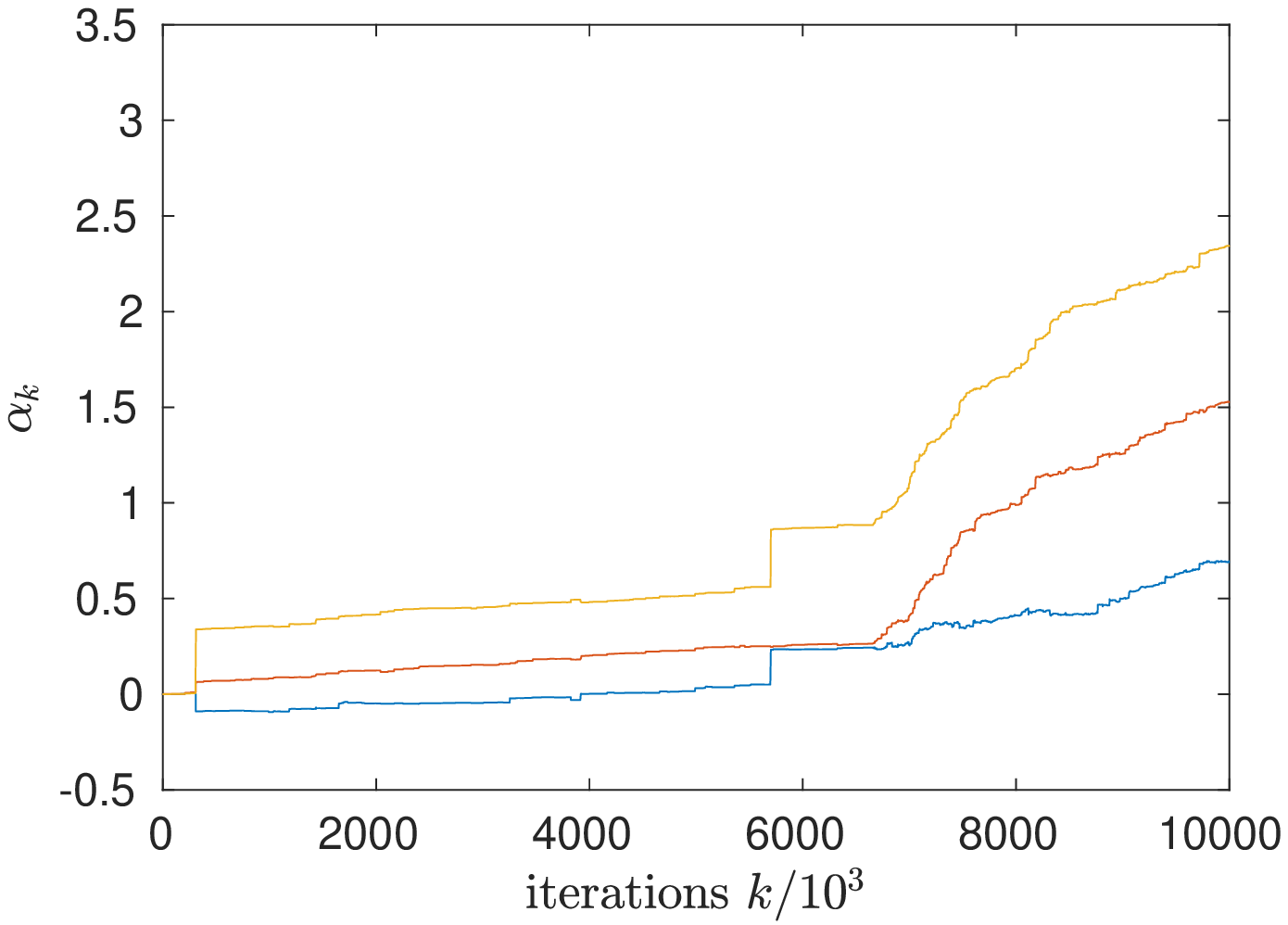}
   \caption{Classical Passive LMS}
    \end{subfigure}
 \caption{Sample paths of estimates $\eth_k \in \reals^3$ of Multi-kernel Passive LMS vs Classical Passive LMS for true parameter $\th^o = [1,2,3]^\p$. For the first 6666 iterations, $\th_k $ are sampled from $\belief(\cdot)$, the density for $\normald(0,50I)$. For the remaining iterations $\th_k $ is sampled from $\belief(\cdot)$, the density for $\normald(0,20I)$.}
  \label{sim1}
\end{figure}

\begin{figure}
  \centering
   \begin{subfigure}{.45\textwidth}
     \includegraphics[scale=0.45]{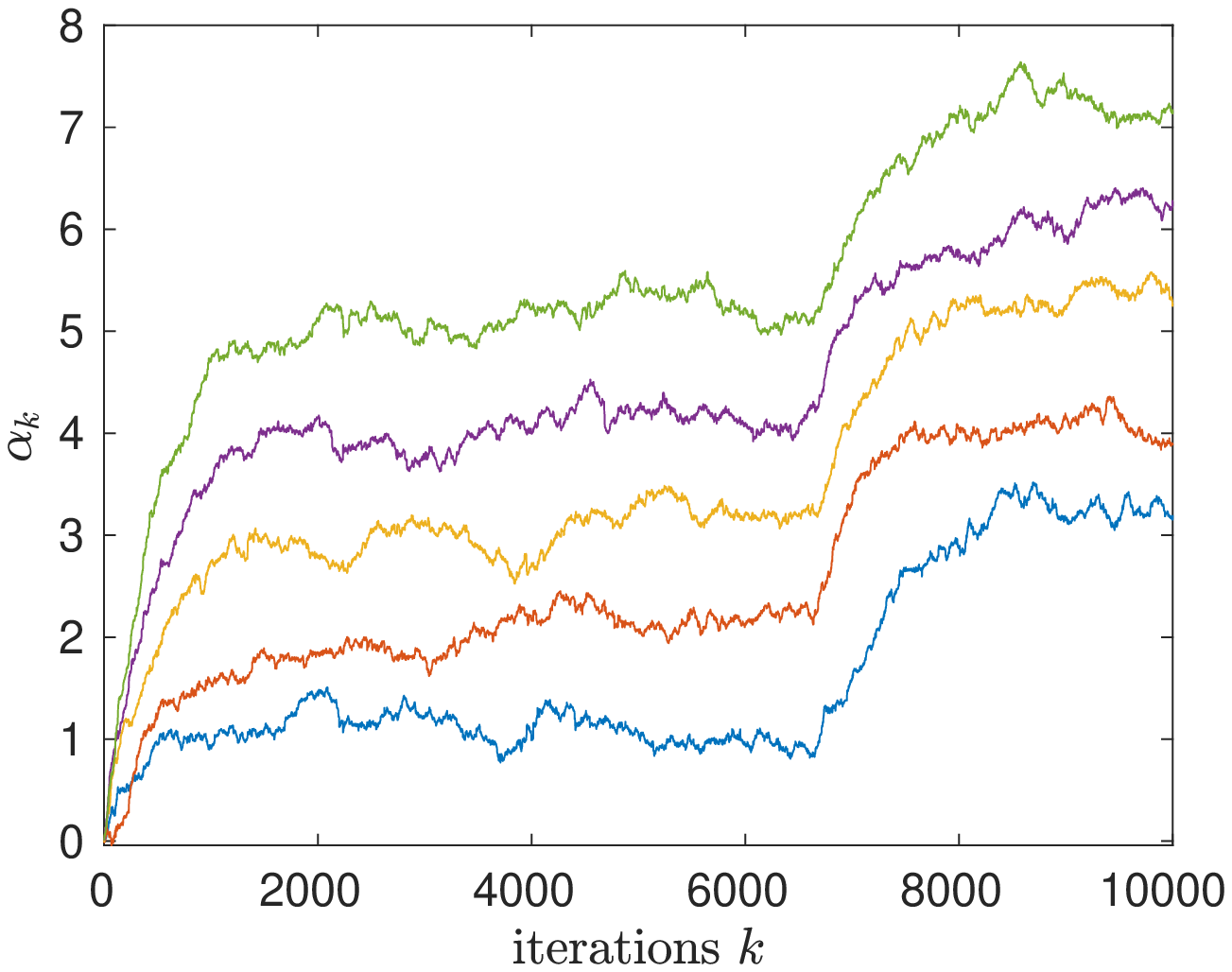}
            \caption{Multi-kernel Passive LMS} \label{multi2a}
    \end{subfigure}
    \begin{subfigure}{.45\textwidth}
      \includegraphics[scale=0.45]{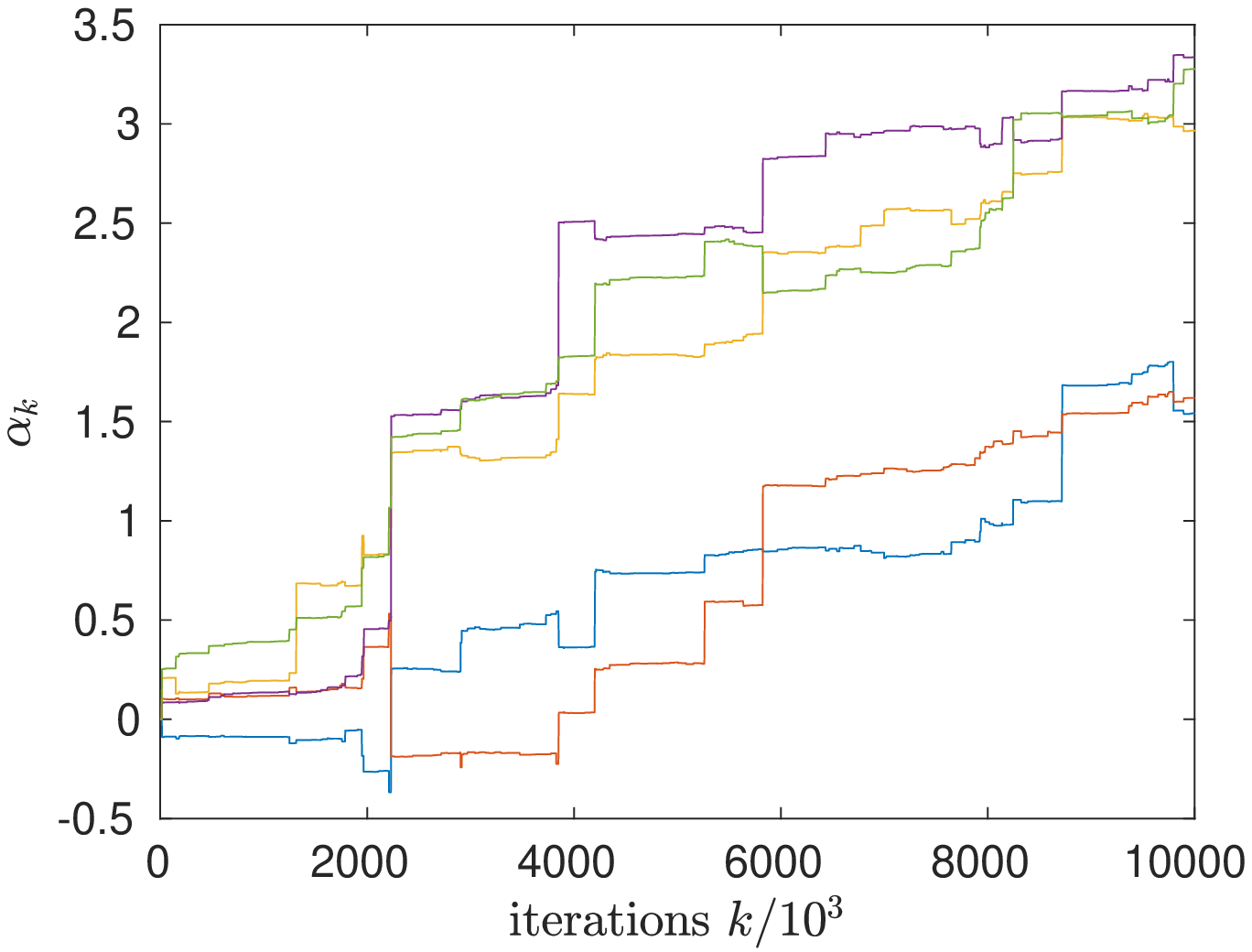}
      \caption{Classical Passive LMS}
    \end{subfigure}
    \caption{Multi-kernel Passive LMS vs Classical Passive LMS for tracking time varying true model. The figure displays the sample path estimates $\eth_k \in \reals^5$. For the first 6666 iterations,
    true parameter $\th^o = [1,2,3,4,5]$  and then $\th^o = [3,4,5,6,7]$. The sampling density is
    $\belief(\th)$ for   $\normald(0,12 I)$.}
  \label{sim2}
\end{figure}

\begin{figure}
  \centering
   \begin{subfigure}{.45\textwidth}
     \includegraphics[scale=0.45]{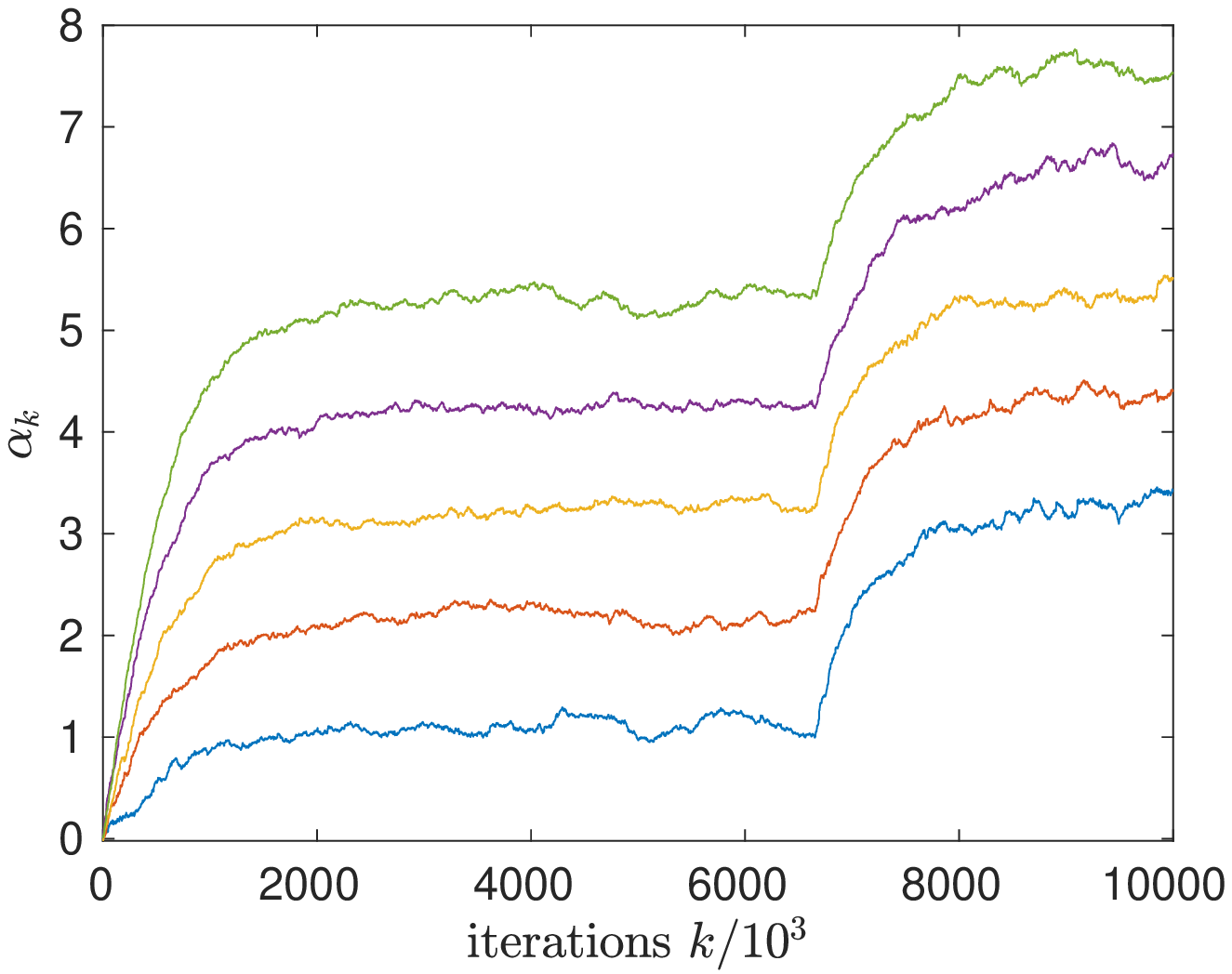}
      \caption{Multi-kernel Passive LMS}
    \end{subfigure}
    \begin{subfigure}{.45\textwidth}
      \includegraphics[scale=0.45]{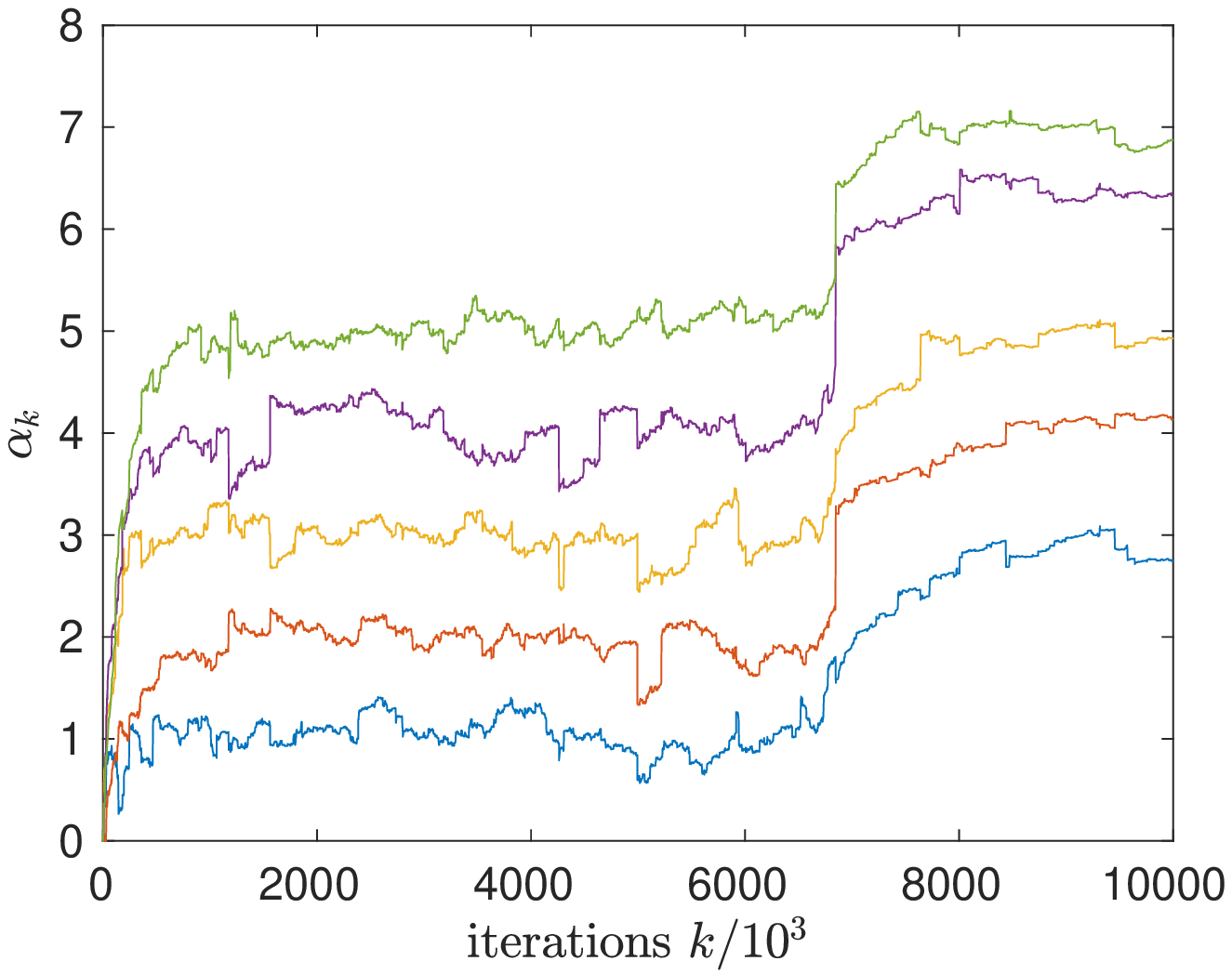}
      \caption{Classical Passive LMS}
    \end{subfigure}
  \caption{Multi-kernel Passive LMS vs Classical Passive LMS for tracking time varying true model. For the first 6666 iterations,
    true parameter $\th^o = [1,2,3,4,5]$  and then $\th^o = [3,4,5,6,7]$. The sampling density is
    $\belief(\th)$ for $ \normald(0,6 I)$.}
  \label{sim3}
\end{figure}

\section{Discussion}

    This paper has presented and analyzed a  multi-kernel two-time scale passive stochastic gradient algorithm.  The proposed algorithm is  a {\em passive} learning algorithm since the gradients are not evaluated at points specified by the algorithm;   instead the gradients are evaluated at the random points  $\th_\dtime$. By observing  noisy measurements of the gradient, the passive algorithm estimates the minimum.
    The proof involves a novel application of the  Bernstein von Mises theorem (which in simple terms is a central limit theorem for a Bayesian estimator) along with weak convergence.
    We also illustrated the performance of the  algorithm numerically in transfer learning involving the passive LMS algorithm.

    As mentioned in the introduction,  in addition to the examples presented here, other  important  applications are inverse reinforcement learning \cite{KY20} and transfer learning for more general systems. It is also worth exploring applications involving asynchronous gradient estimates from agents, e.g.\
    stragglers (slow processing nodes) in coded computation in cloud  computing \cite{KFD18}.

\bibliographystyle{IEEEtran}
\bibliography{$HOME/texstuff/styles/bib/vkm}

\appendix

\section*{Matlab Source Code for multi-kernel passive
  algorithm (\ref{eq:mcmcirl})}
The following Matlab code generates Fig.\ref{multi2a}.

\lstset{language=Matlab,%
  breaklines=true,%
   basicstyle=\footnotesize,
    morekeywords={matlab2tikz},
    keywordstyle=\color{blue},%
    morekeywords=[2]{1}, keywordstyle=[2]{\color{black}},
    identifierstyle=\color{black},%
    stringstyle=\color{mylilas},
    commentstyle=\color{mygreen},%
    showstringspaces=false,
    numbers=left,%
    numberstyle={\tiny \color{black}},
    numbersep=9pt, 
    emph=[1]{for,end,break},emphstyle=[1]\color{red}, 
}

\begin{lstlisting}
%Algorithm parameters
N = 10^4; theta0 = [1 2 3 4 5]'; step =2e-3;   sigma_theta=12; sigma_kernel = 0.2;  L =1000; thdim=size(theta0,1); sf=12; 

est  = zeros(thdim,N); th = zeros(thdim,1); alfa = th;
sigmoidy = zeros(L,1);  sigmoid0 = zeros(L,1);
y = zeros(1,L); 

    for k=1:N
         if k <= 2*N/3
	     theta0 = [1 2 3 4 5]';
	  else
	     theta0 = [3 4 5 6 7]'; % true parameter jump changes
	  end;

% agent computes gradient at a random value of th
	  th =   sigma_theta * randn(thdim,L);  % agent chooses th randomly
	  psi = randn(thdim,L);  % regression vector
	  y =  psi'*theta0 + randn(L,1);

% Multi-kernel algorithm
	  d = vecnorm(th-alfa);
          weight =10^(sf*thdim)*exp(- d/(2*(sigma_kernel)));
	  nweight = weight/sum(weight);
	  wgrad = sum(psi'.* (y' -  sum(psi.*th,1))'.*nweight',1);
	  alfa = alfa +step* wgrad' ;
	  est(:,k) = alfa ;         
      end;

figure(2); plot(est'); % plot estimates
\end{lstlisting}

\begin{IEEEbiography}{Vikram Krishnamurthy}
(F'05) received the Ph.D. degree from the Australian
National University
in 1992. He is  a professor in the
School of Electrical \& Computer Engineering,
Cornell University. From 2002-2016 he
was a Professor and Canada Research Chair
at the University of British Columbia, Canada.
His research interests include statistical signal
processing  and
stochastic control in social networks and adaptive sensing. He served
as Distinguished Lecturer for the IEEE Signal Processing Society and
Editor-in-Chief of the IEEE Journal on Selected Topics in Signal Processing.
In 2013, he was awarded an Honorary Doctorate from KTH
(Royal Institute of Technology), Sweden. He is author of the books
{\em Partially Observed Markov Decision Processes} and
{\em Dynamics of Engineered Artificial Membranes and Biosensors} published by Cambridge
University Press in 2016 and 2018, respectively.
\end{IEEEbiography}

\mbox{}
\vspace{-1.5cm}

\begin{IEEEbiography}{George Yin}
(S'87-M'87-SM'96-F'02) received the B.S. degree in mathematics from
the University of Delaware in 1983, and the M.S. degree in electrical
engineering and the Ph.D. degree in applied mathematics from Brown
University in 1987.
He joined the Department of Mathematics,
Wayne State University in 1987, and became
Professor in 1996 and University Distinguished
Professor in 2017. He moved to the University of Connecticut in 2020. His research interests include
stochastic processes, stochastic systems theory and applications.
Dr. Yin was the Chair of the SIAM Activity Group on Control and
Systems Theory, and served on the Board of Directors of the American
Automatic Control Council. He is the Editor-in-Chief
of {\em SIAM Journal on Control and Optimization}, was a Senior
Editor of {\em IEEE Control Systems Letters}, and is an Associate Editor of {\em ESAIM: Control, Optimisation and Calculus of Variations}, {\em Applied Mathematics and Optimization} and many other journals. He
was an Associate Editor of {\em Automatica} 2005-2011 and {\em IEEE Transactions on Automatic Control} 1994-1998. He is a Fellow of IFAC
and a Fellow of SIAM.
\end{IEEEbiography}

\end{document}

%% file: mycommands.tex
\newcommand{\obsnoise}{\kernelstep}
\newcommand{\weight}{\gamma}

\newcommand{\step}{\varepsilon}
\newcommand{\e}{\step}
\newcommand{\stepa}{\mu}
\newcommand{\kernelstep}{\mu}
\newcommand{\kernel}{K}
\newcommand{\kerneln}{\frac{1}{\kernelstep^\thdim}\,\kernel}

\renewcommand{\th}{\theta}
\renewcommand{\eth}{\alpha}

\newcommand{\Cost}{C}

\newcommand{\cost}{c}

\newcommand{\obs}{y}

\newcommand{\dtime}{k}

\newcommand{\thdim}{N}
\newcommand{\reals}{{\mathbb R}}

\newcommand{\beq}{\begin{equation}}
\newcommand{\eeq}{\end{equation}}

\newcommand{\pdf}{p}

\newcommand{\E}                 {\Bbb{E}}

\newcommand{\horizon}{T}
\newcommand{\belief}{\pi}

\newcommand{\argmin}{\operatornamewithlimits{arg\,min}}

\newcommand{\normald}{\mathbf{N}}
\newcommand{\normal}{\normald}
\newcommand{\p}{\prime}

\newcommand{\thtrue}{{\th^o}}

\newcommand{\snoise}{w}

\newcommand{\al}{\eth}

\newcommand{\wdt}{\widetilde}
\newcommand{\wdh}{\widehat}

\newcommand{\beql}[1]{\begin{equation} \label{#1}}

  \newcommand{\bed}{\begin{displaymath}}
\newcommand{\eed}{\end{displaymath}}
  \newcommand{\bea}{\bed\begin{array}{rl}}
\newcommand{\eea}{\end{array}\eed}
\newcommand{\ad}{&\!\!\!\disp}
\newcommand{\aad}{&\disp}
\def\disp{\displaystyle}
\newcommand{\barray}{\begin{array}{ll}}
                       \newcommand{\earray}{\end{array}}

                     \def\({\Big(}
                     \def\){\Big)}
                     
\newcommand{\cd}{(\cdot)}
\def\op{{\cal L}}
\newcommand{\george}{\footnote}
\def\george#1 {\fbox {\footnote {\ }}\ \footnotetext { From George: #1}}
\def\vikram#1 {\fbox {\footnote {\ }}\ \footnotetext { From Vikram: #1}}

\def \defn {\stackrel{\triangle}{=}}

\newcommand{\numparticles}{L}

\definecolor{mygreen}{RGB}{28,172,0} 
\definecolor{mylilas}{RGB}{170,55,241}

\newcommand{\ones}{\mathbf{1}}

\newtheorem{theorem}            {Theorem}

\newtheorem{corollary}          [theorem]{Corollary}
\newtheorem{proposition}        [theorem]{Proposition}

\newtheorem{lemma}              [theorem]{Lemma}

\newcommand{\cov}{\operatorname{Cov}}

\newenvironment{proof}{\noindent {\bf Proof:}}{\hfill$\square$}

\newcommand{\tr}{{\rm tr}}

\def\op{{\cal L}}
\def\al{\alpha}
\def\th{\theta}

\def\({\Big(}
\def\){\Big)}
\def\disp{\displaystyle}

\def\sumii{\sum_{k\in I_\e}}
\def\para#1{\vskip .25\baselineskip\noindent{\sf #1}}
\def\qed{\hfill{$\Box$}}

\newcommand{\D}{\wdh \nabla}
\newcounter{assum_index}
\newcommand{\mean}{m}
\newcommand{\asympCov}{P}

\newcommand{\clhs}{a}
\newcommand{\crhs} {f}
\newcommand{\lagrange}{\lambda}
\newcommand{\scal}{s}
\newcommand{\sigmagauss}{\sigma}

\newcommand{\Noise}{W}
\renewcommand{\L}{\mathcal{L}}